\documentclass[3p,10pt,a4paper,twoside,fleqn,sort&compress]{filomat}
\usepackage{amssymb,amsmath,latexsym}
\usepackage[varg]{pxfonts}
\usepackage{epsfig,fancybox}
\usepackage{graphicx,subfigure}
\usepackage{ulem}
\usepackage{xcolor}
\newtheorem{theorem}{Theorem}[section]

\newtheorem{corollary}[theorem]{Corollary}

\newtheorem{lemma}[theorem]{Lemma}

\newcommand{\FilVol}{xx}
\newcommand{\FilYear}{yyyy}
\newcommand{\FilPages}{zzz--zzz}
\newcommand{\FilDOI}{(will be added later)}
\newcommand{\FilReceived}{Received: dd Month yyyy; Accepted: dd Month yyyy}
\begin{document}

%\hskip6.1cm 
%\includegraphics{memo.eps}
%\vskip-2.5cm

%%%%%% TO BE ENTERED BY THE AUTHOR(S)
%%%
%%% ENTER TITLE
\title{Modification of Gesture-Determined-Dynamic Function with Consideration of Margins for Motion Planning of Humanoid Robots}

%%% AUTHOR(S) FULL NAMES, AND EMAIL ADDRESSES
\author[affil1,affil2]{Zhijun Zhang}
\ead{auzjzhang@scut.edu.cn}
%%%
\author[affil1,affil2,affil3]{Lingdong Kong}
\ead{ldkong@ieee.org}
%%%
\author[affil1,affil2,affil4]{Yaru Niu}
\ead{yaruniu@gatech.edu}
%%%
\author[affil1]{Ziang Liang}
\ead{auzyliang@mail.scut.edu.cn}
%%% ENTER AUTHOR(S) AFFILIATION(S)
\address[affil1]{School of Automation Science and Engineering, South China University of Technology, Guangzhou, China}
\address[affil2]{Center for Brain Computer Interfaces and Brain Information Processing, South China University of Technology, Guangzhou, China}
\address[affil3]{School of Computer Science and Engineering, Nanyang Technological University, Singapore}
\address[affil4]{School of Electrical and Computer Engineering, Georgia Institute of Technology, Atlanta, GA, USA}
%%% AND CORRESPONDINGLY FOR OTHER AUTHORS, IF THERE ARE MORE AUTHORS
%%% ENTER ABBREVIATED AUTHOR(S) NAMES FOR PAGE HEADINGS
\newcommand{\AuthorNames}{Z. Zhang et al.}
%%% IF THERE ARE MORE THAN TWO AUTHORS WRITE
%%% \newcommand{\AuthorNames}{First Author et al.}
%%%

%%% ENTER MSC, KEYWORDS, RECEIVED, EDITOR, THANKS FOR FINANCIAL SUPPORT FOR RESEARCH
\newcommand{\FilMSC}{Primary xxxxx (mandatory); Secondary xxxxx, xxxxx (optionally)}
\newcommand{\FilKeywords}{Humanoid robots, Motion planning, Quadratic programming, Dynamics, Robotic manipulation}
\newcommand{\FilCommunicated}{(name of the Editor, mandatory)}
\newcommand{\FilSupport}{Research supported in part by the National Key Research and Development Program of China under Grant 2017YFB1002505, in part by the National Natural Science Foundation
of China under Grant 61976096, Grant 61603142, and Grant 61633010, in part by the Guangdong Foundation for Distinguished Young Scholars under Grant 2017A030306009, in part by the Guangdong Special Support Program under Grant 2017TQ04X475, in part by the Science and Technology Program of Guangzhou under Grant 201707010225, in part by the Fundamental Research Funds for Central Universities under Grant 2017MS049, in part by the Scientific Research Starting Foundation of South China University of Technology, in part by the National Key Basic Research Program of China (973 Program) under Grant 2015CB351703, in part by the Italian Ministry of Education, University and Research under the Project “Department of Excellence LIS4.0—Lightweight and Smart Structures for Industry 4.0,” in part by the Guangdong Key Research and Development Program under Grant 2018B030339001, and in part by Guangdong Natural Science Foundation Research Team Program 1414060000024.}
%%% If you do not want to thank for the financial support of the research, remove 
%%% the previous line (i.e., leave \FilSupport undefined)
%%%%%%%%%%%%%%%%%%%%%%%%%%%%%%%%%%%%%%%%%%%%%%%%%%

\begin{abstract}
The gesture-determined-dynamic function (GDDF) offers an effective way to handle the control problems of humanoid robots. Specifically, GDDF is utilized to constrain the movements of dual arms of humanoid robots and steer specific gestures to conduct demanding tasks under certain conditions. However, there is still a deficiency in this scheme. Through experiments, we found that the joints of the dual arms, which can be regarded as the redundant manipulators, could exceed their limits slightly at the joint angle level. The performance straightly depends on the parameters designed beforehand for the GDDF, which causes a lack of adaptability to the practical applications of this method. In this paper, a modified scheme of GDDF with consideration of margins (MGDDF) is proposed. This MGDDF scheme is based on quadratic programming (QP) framework, which is widely applied to solving the redundancy resolution problems of robot arms. Moreover, three margins are introduced in the proposed MGDDF scheme to avoid joint limits. With consideration of these margins, the joints of manipulators of the humanoid robots will not exceed their limits, and the potential damages which might be caused by exceeding limits will be completely avoided. Computer simulations conducted on MATLAB further verify the feasibility and superiority of the proposed MGDDF scheme.
\end{abstract}

\maketitle

%%%%%% THIS PART MUST BE PLACED IMMEDIATELY AFTER THE \maketitle COMMAND
%%%%%% BACK TO ORIGINAL FOOTNOTES
\makeatletter
\renewcommand\@makefnmark%
{\mbox{\textsuperscript{\normalfont\@thefnmark)}}}
\makeatother
%%%%%%

%%%%%% Section I: Introduction
\section{Introduction}

A robot is a programmable machine capable of executing complex tasks automatically. Traditionally, robots are designed to perform mechanical actions with no regard to how they look. With the evolution of the robotic industry, as well as the higher psychological requirement of human beings, more and more robots are designed and conducted to appear and behave like human beings \cite{paper.1}. Such similarity can be for functional purposes, for example, interacting with human instruments and facilities, or for other experimental purposes \cite{paper.2,paper.3}. In general, these humanoid robots are endowed with similar body structures of human beings, i.e., a head, a torso, two arms, and two legs \cite{paper.4}. Some of them also have their head parts designed to replicate emotional expressions of the human face, such as eye-blinking and lip-sipping \cite{paper.5}. An automatic perception for facial expressions of humanoid robots proposed by Zhang $et \ al.$ demonstrated the potential in developing personalized robots with social intelligence \cite{paper.6}. Trovato $et \ al.$ created a facial expression generator applied to a KOBIAN-R robot to establish an adaptive system based on human communication and facial anatomy \cite{paper.7}. Such a system is effective for the implementation of complex interactions between humans and robots. For a XIN-REN robot, a facial expression learning method based on energy conservation principle proposed by Ren $et \ al.$ successfully ensured smoother trajectories for multi-frame imitation of humanoid robots \cite{paper.8}.

In addition to the head, the arms are also crucial and indispensable for a humanoid robot. Recently, various novel humanoid robots equipped with arms have been studied and further applied in home service, personal entertainment, inspection, health care, manufacturing, and so on \cite{article.2,article.3,article.4,article.5}. The dual arms of humanoid robots can not only conduct tasks \cite{article.6}, but also convey emotions with body language \cite{article.7, article.8}. To perform daily tasks with dual-arms of humanoid robots, the motion planning problem should be considered. Aly $et \ al.$ proposed a system to generate the dynamic characteristics of gestures and postures of the NAO robot during nonverbal communications \cite{paper.9,paper.10}. Tondu $et \ al.$ designed an anthropomorphic robot arm and its ancillary motion control method \cite{paper.11}.
One of the basic problems of dual-arm motion planning is the inverse kinematic problem \cite{article.9}, i.e., given the trajectories of end-effector, computing joint variables at each time instant \cite{book.8}. Most dual arms of humanoid robots have more than three degrees of freedom (termed as DOF). This redundancy improves the flexibility of dual arms when the humanoid robots implement end-effector tasks \cite{article.11}. However, the redundant DOF also inevitably increases the difficulties of computation \cite{article.12}.

The traditional redundancy resolution method is pseudo-inverse method \cite{article.13}. Wang $et \ al.$ proposed a closed-loop method for solving the inverse kinematics in Ref. \cite{article.14}, which was based on pseudo-inverse method, to control the dual arms on a mobile platform. In addition, the inverse of matrices has to be considered when utilizing the pseudo-inverse method to solving inverse kinematics, which can be challenging \cite{article.15}. Quadratic programming (termed as QP) methods are preferred recently, and a QP-based task priority framework is proposed in Ref. \cite{article.16} to resolve the kinematic redundancy problem. Inspired by the aforementioned optimization method, a QP-based online generation scheme for generating expected gestures of dual arms is proposed by Zhang $et \ al.$ \cite{article.1}, and a gesture-determined-dynamic function (termed as GDDF) is designed and introduced to the kinematics of humanoid robots. This GDDF scheme can not only handle the redundancy resolution problem within the joint limits, but also dynamically adjust gestures to the desired position. Different from the exiting pseudo-inverse method \cite{article.17,article.18,article.19} or very few QP-based methods focusing on single arm \cite{article.20,article.21,article.22}, together with the numerical QP solver, the GDDF scheme can generate dual-arm behaviors for humanoid robots \cite{article.23,article.24,article.26,article.28} in 3-D space online.

However, there is still a deficiency in the GDDF scheme. On the basis of a series of experiments, we find that at some parameter groups during execution, the joint might slightly exceed its limit. That is to say, if the parameter is falsely selected, dead-lock of robot arms or even damages could occur, which is perilous for the practical applications. To remedy this disadvantage, a modified GDDF method named MGDDF is proposed in this paper. With consideration of three kinds of margins (i.e., MGDDF-1, MGDDF-2 and MGDDF-3, which will be carefully discussed in the following sections) are introduced to the MGDDF method to constrain the movement of the joints.

The remainder of this paper is listed as follows. The overviews of related works are presented in Section II. Section III gives the preliminary and problem formulation of dual arms. The traditional GDDF method and the proposed MGDDF method are discussed in Section IV. Simulations and experiments are made in Section V to verify the effectiveness of the proposed method. Section VI offers a concluding remarks.

The main contributions of this paper can be summarized as follows:
\begin{itemize}
\item With consideration of three margins, a modified scheme named MGDDF is proposed for solving kinematic problems of dual arms of humanoid robot.
\item The gesture of dual arms can be smoothly achieved by using MGDDF method, and the joints would not exceed their limits during the whole execution of tasks.
\end{itemize}

%%%%%% Section II: Related Works
\section{Related Works}
The motion planning of humanoid robots have been studied for a few decades. Gouda $et \ al.$ proposed a whole-body motion planning approach of NAO robot using onboard sensing and achieved reliable motion sequences when operating in complex environment \cite{related.1}. A path planning method which computes dynamics and collision avoidance for full-body gestures of humanoid robots was proposed by Kagami $et \ al.$ \cite{related.2}. This method contains a filtering function that constrains the zero moment point of each trajectory to obtain statically stabilization for the entire body. Without lost of generality, Dalibard $et \ al.$ presented a collision-free method for whole-body working control of humanoid robots \cite{related.3}. In Ref. \cite{related.4}, Ohashi $et \ al.$ introduced a linear inverted pendulum mode and designed a collision avoidance method for a robot with an upper body to discriminate the start/stop of force of robot arms. On the basis of Lyapunov stability theory, an adaptive control method was proposed by Liu $et \ al.$ to satisfy the coordination of dual arms of humanoid robots \cite{related.5}. Simulation results prove that when applying this method, the internal force of the arms can be held and the position errors are arbitrarily small.

As for the motion optimization, Ayusawa $et \ al.$ proposed a novel re-targeting method with geometric parameter identification, motion planning and inverse kinematics to handle control problems of humanoid robots \cite{related.6}. A practical experiment performed on HRP-4 robot further verified the excellent performance of this method. In Ref. \cite{related.7}, an optimization-based approach was proposed by Schulman $et \ al.$ to find collision-free trajectories of a 18 degrees-of-freedom robot with dual arms. Hong $et \ al.$ who focused on real-time pattern generation utilized the linear pendulum model to solve the walking control problem of MAHRU-R robot with feed-forward and feed-back controller \cite{related.8}. The desired task hierarchy which is based on quadratic programming was solve by Liu $et \ al.$ to reduce the risk of instability when humanoid robots conducting operations \cite{related.9}. According to HRP-2 robot, a novel plan for foot placements was proposed by Kanoun $et \ al.$ to formulate the deformation problem as the inverse kinematic problem and to further be described as a locomotion phase of the desired tasks. \cite{related.10}. For the posture control problem, Choi $et \ al.$ proposed a kinematic resolution method based on the center of mass Jacobian with embedded motion of humanoid robots \cite{related.11}. A motion planning using swept volume approximation was researched by Perrin $et \ al.$ to achieve the 3-D obstacle avoidance \cite{related.12}.

In the field of gesture generation, Yanik $et \ al.$ proposed a gesture recognition scheme for path shaping of simulated robot \cite{new.0}. They first extracted gesture data from sequential skeletal depth and then clustered them by a growing neural gas algorithm. Cheng $et \ al.$ presented a hand-gesture-recognition method which can track gestures of a
dance robot in real-time \cite{new.1}. The hand gestures were recorded as signals detected by a 3-axis sensor and represented as discrete cosine transform coefficients. Furthermore, a artificial communicator for gesture recognition, proposed by Zhang $et \ al.$, was utilized for human-robot cooperative assembly \cite{new.2}. In Ref. \cite{new.3}, the tracking control of manipulators was solved by nonlinear observers which provided joint coordinate rates. Arteaga $et \ al.$ utilized this method by associating Euler-Lagrange equations to describe gesture motions and achieved high control accuracy.

%%%%%% Section III: Preliminary and Problem Formulation
\section{Preliminary and Problem Formulation}\label{sec.model}

The forward kinematics of dual arms of humanoid robots can be described as
\begin{equation}
\textit{\textbf{p}}_L=\Im_L(\mathbf{\theta}_L), \ \textit{\textbf{p}}_R=\Im_R(\mathbf{\theta}_R)\label{equ.timeinv}
\end{equation}
where $\textit{\textbf{p}}_L$ and $\textit{\textbf{p}}_R$ denote the position vectors of the end effector. $\theta_L$ and $\theta_R$ stand for the vectors of left and right joints. $\Im_L(\cdot)$ and $\Im_R(\cdot)$ are continuous nonlinear functions. These two equations can be obtained when the manipulator is given. For redundant manipulator, the inverse kinematic problem is basic and the motion states of each joint should be solved according to the trajectories of the end effector.

Inspired by previous work in Ref. \cite{article.1}, an online optimization algorithm to solve the redundancy problem is presented. The following equations at velocity level are obtained
\begin{eqnarray}
\frac{\dot{\theta}_L^\text{T}\textit{\textbf{A}}\dot{\theta}_L}{2}+\frac{\dot{\theta}_R^\text{T}\textit{\textbf{B}}\dot{\theta}_R}{2}
+\textit{\textbf{c}}_L^\text{T}\theta_L+\textit{\textbf{c}}_R^\text{T}\theta_R\label{eq4}\\
\textit{\textbf{J}}_L(\theta_L)\dot{\theta}_L=\dot{\textit{\textbf{p}}}_L+\wp_L\big(\textit{\textbf{p}}_L-\Im_L(\theta_L)\big)\label{eq5}\\
\textit{\textbf{J}}_R(\theta_R)\dot{\theta}_R=\dot{\textit{\textbf{p}}}_R+\wp_R\big(\textit{\textbf{p}}_R-\Im_R(\theta_R)\big)\label{eq6}\\
\theta_L^-\leqslant\theta_L\leqslant\theta_L^+\label{eq7}\\
\theta_R^-\leqslant\theta_R\leqslant\theta_R^+\label{eq8}\\
\dot{\theta}_L^-\leqslant\dot{\theta}_L\leqslant\dot{\theta}_L^+\label{eq9}\\
\dot{\theta}_R^-\leqslant\dot{\theta}_R\leqslant\dot{\theta}_R^+\label{eq10}
\end{eqnarray}
where superscript $^\text{T}$ denotes the transpose of a vector or a matrix. $\dot{\theta}_{L}$ and $\dot{\theta}_{R}\in \mathbb{R}^n$ denote joint-velocity vectors of left and right arms. $\textit{\textbf{A}}\in \mathbb{R}^{n\times n}$ and $\textit{\textbf{B}}\in \mathbb{R}^{n\times n}$ are coefficient matrices of quadratic terms. $\textit{\textbf{c}}_L$ and $\textit{\textbf{c}}_R$ are coefficient vectors related to linear terms. $\textit{\textbf{J}}_L(\cdot)$ and $\textit{\textbf{J}}_R(\cdot)\in \mathbb{R}^{n\times m}$ are the Jacobian matrices of the left and right arms, respectively. $\dot{\textit{\textbf{p}}}_{L}$ and $\dot{\textit{\textbf{p}}}_{R}\in \mathbb{R}^n$ denote velocity vectors of the end effector. The position-error feedbacks $\wp_L\big(\textit{\textbf{p}}_L-\Im_L(\theta_L)\big)$ and $\wp_R\big(\textit{\textbf{p}}_R-\Im_R(\theta_R)\big)$ are considered with non-negative parameters $\wp_L$ and $\wp_R$. The series of inequalities (\ref{eq7})-(\ref{eq10}) are the constraints of joints and joint-velocities. $\theta_L^\pm$ and $\theta_R^\pm$ denote physical joint limits, while $\dot{\theta}_L^\pm$ and $\dot{\theta}_R^\pm$ denote joint-velocity limits, respectively.

The aforementioned dual formulas (\ref{eq5})-(\ref{eq10}) can be combined as matrix forms to offer a much more clear expression, i.e.,

$\bullet$ position-error feedback equalities (\ref{eq5})-(\ref{eq6}) are combined as
\begin{equation}
\begin{bmatrix}
\textit{\textbf{J}}_\text{L} &\textbf{\text{0}}_{m\times n}\\
\textbf{\text{0}}_{m\times n} &\textit{\textbf{J}}_\text{R}
\end{bmatrix}
\cdot
\begin{bmatrix}
\dot{\theta}_L\\
\dot{\theta}_R
\end{bmatrix}
=
\begin{bmatrix}
\dot{\textit{\textbf{p}}}_L\\
\dot{\textit{\textbf{p}}}_R
\end{bmatrix}
+
\begin{bmatrix}
\wp_L\big(\textit{\textbf{p}}_L-\Im_L(\theta_L)\big)\\
\wp_R\big(\textit{\textbf{p}}_R-\Im_R(\theta_R)\big)
\end{bmatrix};
\label{eqn.c1}
\end{equation}

$\bullet$ joint limit inequalities (\ref{eq7})-(\ref{eq8}) are combined as
\begin{eqnarray}
\begin{bmatrix}
\theta_L^-\\
\theta_R^-
\end{bmatrix}
\leqslant
\begin{bmatrix}
\theta_L\\
\theta_R
\end{bmatrix}
\leqslant
\begin{bmatrix}
\theta_L^+\\
\theta_R^+
\end{bmatrix}
\in \mathbb{R}^{2n};
\label{eqn.c2}
\end{eqnarray}

$\bullet$ joint-velocity limit inequalities (\ref{eq9})-(\ref{eq10}) are combined as
\begin{eqnarray}
\begin{bmatrix}
\dot{\theta}_L^-\\
\dot{\theta}_R^-
\end{bmatrix}
\leqslant
\begin{bmatrix}
\dot{\theta}_L\\
\dot{\theta}_R
\end{bmatrix}
\leqslant
\begin{bmatrix}
\dot{\theta}_L^+\\
\dot{\theta}_R^+
\end{bmatrix}
\in \mathbb{R}^{2n}.
\label{eqn.c3}
\end{eqnarray}

With the consideration of objective function (\ref{eq4}) and modified constrains (\ref{eqn.c1})-(\ref{eqn.c3}), the following standard QP can be formulated as follows
\begin{eqnarray}
\text{min.} \ ~
\frac{\dot{\Theta}^\text{T}\textit{\textbf{M}}\dot{\Theta}}{2}+\textit{\textbf{c}}^\text{T}\Theta\label{eq14}~~~~~~~~~~~~~~~~~~ \\
\text{s. t.} \ ~
\textit{\textbf{J}}(\Theta)\dot{\Theta}=\dot{\Upsilon}+\wp(\Upsilon-\Im(\Theta))\\
\Theta^-\leqslant\Theta\leqslant\Theta^+\label{eq16}\\
\dot{\Theta}^-\leqslant\dot{\Theta}\leqslant\dot{\Theta}^+\label{eq17}
\end{eqnarray}
where
$\Theta=[\theta_L^\text{T}, \theta_R^\text{T}]$, $\textit{\textbf{c}}=[\textit{\textbf{c}}_L^\text{T}, \textit{\textbf{c}}_R^\text{T}]$,
$\Theta^-=[\theta_L^{-\text{T}}, \theta_R^{-\text{T}}]^\text{T}$,
$\Theta^+=[\theta_L^{+\text{T}}, \theta_R^{+\text{T}}]^\text{T}$,
$\dot{\Theta}=\text{d}\Theta/\text{d}t=[\dot{\theta}_L^\text{T}, \dot{\theta}_R^\text{T}]$,
$\dot{\Theta}^-=[\dot{\theta}_L^{-\text{T}}, \dot{\theta}_R^{-\text{T}}]^\text{T}$,
$\dot{\Theta}^+=[\dot{\theta}_L^{+\text{T}}, \dot{\theta}_R^{+\text{T}}]^\text{T}$,
$\dot{\Upsilon}=[\dot{\textit{\textbf{p}}}_L^\text{T}, \dot{\textit{\textbf{p}}}_R^\text{T}]^\text{T}$,
\begin{equation}
\textit{\textbf{M}}=
\begin{bmatrix}
\textit{\textbf{A}} &\textbf{\text{0}}_{n\times n}\\
\textbf{\text{0}}_{n\times n} &\textit{\textbf{B}}
\end{bmatrix}
\in \mathbb{R}^{2n\times 2n},
\end{equation}
\begin{equation}
\textit{\textbf{J}}=
\begin{bmatrix}
\textit{\textbf{J}}_L &\textbf{\text{0}}_{n\times n}\\
\textbf{\text{0}}_{n\times n} &\textit{\textbf{J}}_R
\end{bmatrix}
\in \mathbb{R}^{2m\times 2n},
\end{equation}
\begin{equation}
\wp=
\begin{bmatrix}
\wp_L &\textbf{\text{0}}_{m\times m}\\
\textbf{\text{0}}_{m\times m} &\wp_R
\end{bmatrix}
\in \mathbb{R}^{2m\times 2m}.
\end{equation}

It is worth pointing out that vector $\textit{\textbf{c}}$ and matrix $\textit{\textbf{M}}$ are determined by a definite redundancy-resolution scheme. Matrix $\wp$ stands for the feedback gain and is determined by practical effect. Different goals can be obtained by choosing different criteria of vector $\textit{\textbf{c}}$ and matrix $\textit{\textbf{M}}$. The following three schemes are introduced, i.e.,

$\bullet$ MKE scheme:

If vector $\textit{\textbf{c}}=[\textbf{\text{0}}; \textbf{\text{0}}]$, and matrix $\textit{\textbf{M}}=[\mathcal{I}_L, 0; 0, \mathcal{I}_R]$ (where $\mathcal{I}$ denotes inertia matrix), then the QP problem (\ref{eq14})-(\ref{eq17}) forms a minimum-kinetic-energy (MKE) scheme.

$\bullet$ RMP scheme:

If vector $\textit{\textbf{c}}=[\lambda(\theta_L-\theta_L(0)); \lambda(\theta_R-\theta_R(0))]$ (where $\lambda$ is a non-negative parameter), and matrix $\textit{\textbf{M}}=\textit{\textbf{I}}$ is an identity matrix, then the QP problem (\ref{eq14})-(\ref{eq17}) forms a repetitive motion planning (RMP) scheme.

$\bullet$ MVN scheme:

If vector $\textit{\textbf{c}}=[\textbf{\text{0}}$, and matrix $\textit{\textbf{M}}=\textit{\textbf{I}}$ is an identity matrix, then the QP problem (\ref{eq14})-(\ref{eq17}) forms a minimum-velocity-norm (MVN) scheme.

%%%%%% Section IV: Proposed Methods
\section{Proposed Methods}

Humanoid robots are not only asked to finish the end-effector tasks, but also are demanded to ``behave'' naturally as a human \cite{paper.8,article.23,article.24}. Being the most straight way of expressing feelings and emotions, gestures of humanoid robots have been widely designed and exerted to accomplish natural human-robot interactions \cite{article.26,article.23,article.28}. A gesture-determined-dynamic function (GDDF) is presented in Ref. \cite{article.1}. Associating with a QP framework (\ref{eq14})-(\ref{eq17}) and the GDDF method, the researchers expect to cause expected gestures of humanoid robots.

\subsection{GDDF Scheme}

To generate the expected movement and to execute the end-effector tasks, some joints of the dual arms should be dynamically adjustable \cite{article.1,article.8,article.9,myarticle.11,myarticle.21,myarticle.31,myarticle.41}. The joint limits and joint-velocity limits have already been formulated into QP problem (\ref{eq14})-(\ref{eq17}). A dynamic function which can change the upper and the lower bounds of the limits should be found, and the bounds should be related to the expectation values. Besides, the generated curves should be smooth during the executing process. The following function is then designed to satisfy the demand and the joints can adjust to the expected gesture, i.e.,
\begin{equation}
\tilde{\Theta}^\pm(t)=\Theta^\pm+\frac{\Delta\Theta^\pm}{1+e^{-(t-\tau)/\varrho}}\label{fun18}
\end{equation}
where $\Delta\Theta^\pm=\Theta_{g}^\pm-\Theta^\pm$, and $\Theta_{g}=[\Theta_{gL}^\text{T}; \Theta_{gR}^\text{T}]$ sets the goal configuration of joint. $\varrho$ is a parameter affecting the variation trend. Parameter $\tau=T_d/N$, with $T_d$ denoting the task execution time and $N\geqslant 1$ denoting the parameter affecting the proximity between the adjusted value and the actual value (set point). $\Delta\Theta$ determines whether $\tilde{\Theta}$ increases or decreases. Normally speaking $\tilde{\Theta}^+$ decreases while $\tilde{\Theta}^-$ increases. Fraction $1/[1+e^{-(t-\tau)/\varrho}]$ is a gradual and smooth function and it influences the curve shape of the joint limit. When time $t$ approaches infinite, $e^{-(t-\tau)/\varrho}$ will approach $0$, thus the fraction will approach $1$, and $\tilde{\Theta}^\pm$ will approach $\Theta^\pm+\Delta\Theta^\pm=\Theta_{g}^\pm$.

Function (\ref{fun18}) can lead to the expected movement smoothly and slowly. On the basis of smooth function (\ref{fun18}) and QP-based framework (\ref{eq14})-(\ref{eq17}), the GDDF scheme can be successfully established. The upper and lower bounds of the inequalities will constrain the movement of the joints, and make the joints move just as expected. For dual arms of humanoid robots, all the joints can be adjusted to the tasks at the same time. As for each \emph{i}-th joint, we can set $\Theta_{g\emph{i}}^+=\Theta^+$ and $\Theta_{g\emph{i}}^-=\Theta^-$, and the corresponding joints would reach the target values equally. To achieve the joints' expected gestures, a series of values of parameters $\varrho$, $\tau$ and $N$ which influence the proximity between the adjusted value and the actual value can be easily set.

\subsection{MGDDF Scheme with Consideration of Margins}
The gestures can be smoothly achieved by using GDDF method. However, there is still a deficiency in this scheme. Through the experiment we find that the joint might exceed its limit slightly at some parameter groups. To remedy this problem, three kinds of margins are introduced to constrain the movement of the joints. This modified GDDF method is termed as MGDDF method.

$\bullet$ MGDDF-1:

%%%%%% FIG-1
\begin{figure*}
  \centering
  \includegraphics[width=.99\textwidth]{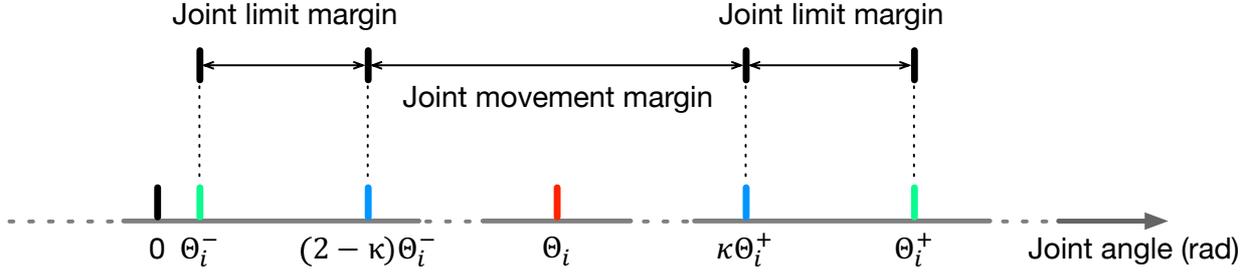}\\
  \caption{Graphical representation of the margins of MGDDF-1. Both upper bound and lower bound of joint limits are greater than zero.}\label{fig.MGDDF-1}
\end{figure*}
%%%%%% FIG-1

Firstly, joint limit (\ref{eq16}) is rewritten as
\begin{equation}
\tilde{\Theta}^-\leqslant\Theta\leqslant\tilde{\Theta}^+
\label{eqn.MGDDF-1}
\end{equation}

Secondly, a margin set at velocity level is considered, and the joint limit constraint (\ref{eqn.MGDDF-1}) is modified as follows
\begin{equation}
\nu\big((2-\kappa)\tilde{\Theta}^-(t)-\Theta\big)\leqslant\dot{\Theta}\leqslant\nu\big(\kappa\tilde{\Theta}^+(t)-\Theta\big)
\label{eqn.cons}
\end{equation}
where $\dot{\Theta}$ denote the joint velocity and $\nu>0$ is a parameter which is applied to scale the feasible region of $\dot{\Theta}$. For simplicity, $\nu$ is set as $\nu=2$ in the following experiments. The critical coefficient $\kappa\in(0,1)$ is selected to define a critical region $(\Theta_i^-, (2-\kappa)\Theta_i^-]$ and $[\kappa\Theta_i^+, \Theta_i^+)$, which is shown in Fig. \ref{fig.MGDDF-1}. It is worth noting that, MGDDF-1 is suitable for the case that the upper bound and lower bound of joint limits are greater than zero.

Moreover, for the \emph{i}-th joint, the joint-velocity constraint of MGDDF-1 are be formulated as follows
\begin{eqnarray}
\begin{split}
\text{max}\big\{\dot{\Theta}^-,\nu\big((2-\kappa)\tilde{\Theta}_{\emph{i}}^-(t)-\Theta_{\emph{i}}\big)\big\}&\leq\dot{\Theta}_{\emph{i}}\\
\text{min}\big\{\dot{\Theta}^+,\nu\big(\kappa\tilde{\Theta}_{\emph{i}}^+(t)-\Theta_{\emph{i}}\big)\big\}&\geq\dot{\Theta}_{\emph{i}}.
\end{split}
\label{eqn.v-MGDDF-1}
\end{eqnarray}

To simplify the expression of joint-velocity constraint (\ref{eqn.v-MGDDF-1}), the following equations are introduced, i.e.,
\begin{eqnarray}
\tilde{\zeta}_{\emph{i}}^-(t)=\text{max}\big\{\dot{\Theta}^-,\nu\big((2-\kappa)\tilde{\Theta}_{\emph{i}}^-(t)-\Theta_{\emph{i}}\big)\big\}\\
\tilde{\zeta}_{\emph{i}}^+(t)=\text{min}\big\{\dot{\Theta}^+,\nu\big(\kappa\tilde{\Theta}_{\emph{i}}^+(t)-\Theta_{\emph{i}}\big)\big\}.
\end{eqnarray}

$\bullet$ MGDDF-2:

%%%%%% FIG-2
\begin{figure*}
  \centering
  \includegraphics[width=.99\textwidth]{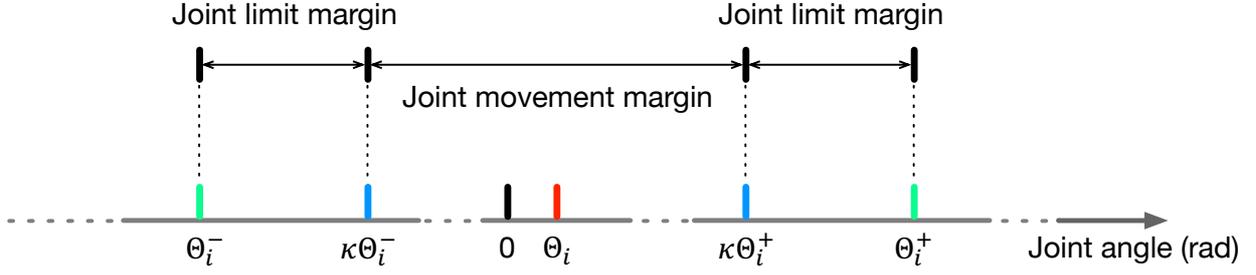}\\
  \caption{Graphical representation of the margins of MGDDF-2. The upper bound of joint limits is greater than zero, while the lower  bound of joint limits is less than zero.}\label{fig.MGDDF-2}
\end{figure*}
%%%%%% FIG-2

Considering the case that upper limit $>0$ and lower limit $<0$. On the basis of constraint (\ref{eqn.cons}), the critical region of MGDDF-2 is  $(\Theta_i^-, \kappa\Theta_i^-]$ and $[\kappa\Theta_i^+, \Theta_i^+)$. The graphical representation of the margins of MGDDF-2 is depicted in Fig. \ref{fig.MGDDF-1}. Similar to MGDDF-1, the joint-velocity constraint of MGDDF-2 is simplified as the following relationships, i.e.,
\begin{eqnarray}
\tilde{\zeta}_{\emph{i}}^-(t)=\text{max}\big\{\dot{\Theta}^-,\nu\big(\kappa\tilde{\Theta}_{\emph{i}}^-(t)-\Theta_{\emph{i}}\big)\big\}\\ \tilde{\zeta}_{\emph{i}}^+(t)=\text{min}\big\{\dot{\Theta}^+,\nu\big(\kappa\tilde{\Theta}_{\emph{i}}^+(t)-\Theta_{\emph{i}}\big)\big\}.
\end{eqnarray}

$\bullet$ MGDDF-3:

When both the upper bound and lower bound of limits are less than zero, the critical region of MGDDF-3 are $(\Theta_i^-, \kappa\Theta_i^-]$ and $[(2-\kappa)\Theta_i^+, \Theta_i^+)$. Sharing the same process, the following result is obtained, i.e.,
\begin{eqnarray}
\tilde{\zeta}_{\emph{i}}^-(t)=\text{max}\big\{\dot{\Theta}^-,\nu\big(\tilde{\Theta}_{\emph{i}}^-(t)-\Theta_{\emph{i}}\big)\big\}\\
\tilde{\zeta}_{\emph{i}}^+(t)=\text{min}\big\{\dot{\Theta}^+,\nu\big((2-\kappa)\tilde{\Theta}_{\emph{i}}^+(t)-\Theta_{\emph{i}}\big)\big\}.
\end{eqnarray}

Combining the aforementioned MGDDF-1, MGDDF-2 and MGDDF-3, the MGDDF scheme with consideration of margins is obtained, i.e.,
\begin{eqnarray}
\text{min.} \ ~
\frac{\dot{\Theta}^\text{T}\textit{\textbf{M}}\dot{\Theta}}{2}+\textit{\textbf{c}}^\text{T}\Theta\label{eq32}~~~~~~~~~~~~~~~~~~ \\
\text{s. t.} \ ~
\textit{\textbf{J}}(\Theta)\dot{\Theta}=\dot{\Upsilon}+\wp(\Upsilon-\Im(\Theta))\label{eq33}\\
\tilde{\zeta}_{\emph{i}}^-(t)\leqslant\dot{\Theta}\leqslant\tilde{\zeta}_{\emph{i}}^+(t)\label{eq34}
\end{eqnarray}

%%%%%% FIG-3
\begin{figure*}
  \centering
  \includegraphics[width=.99\textwidth]{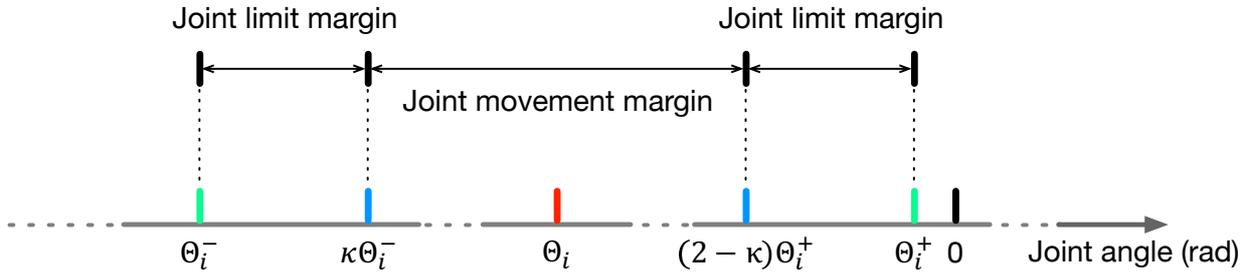}\\
  \caption{Graphical representation of the margins of MGDDF-3. Both upper bound and lower bound of joint limits are less than zero.}\label{fig.MGDDF-3}
\end{figure*}
%%%%%% FIG-3

In summary, the kinematic task of the dual arms of the humanoid robot can be successfully finished by using the proposed MGDDF scheme (\ref{eq32})-(\ref{eq34}). The end-effector velocities $\dot{\textit{\textbf{p}}}_L$ and $\dot{\textit{\textbf{p}}}_R$ are integrated into equation (\ref{eq33}). The key function (\ref{fun18}) determines the upper and lower bounds of inequality (\ref{eq34}). It is worth mentioning that the proposed MGDDF scheme can be solved by a discrete QP solver, which will be discussed in the following sub-section.

\subsection{MGDDF Solvers}

According to Refs. \cite{article.12, article.15}, a linear-variational-inequality (LVI) is introduced to solve MGDDF scheme (\ref{eq32})-(\ref{eq34}). Equivalently, a linear projection equation is led to replace these equations, i.e.,
\begin{equation}
\Phi_{\Omega}\big(\textit{\textbf{d}}-(\Gamma \textit{\textbf{d}}+\textit{\textbf{q}})\big)-\textit{\textbf{d}}=0
\end{equation}
where $\Phi_{\Omega}(\cdot)$ ($\mathbb{R}^{2n+2m}\rightarrow\Omega$) is a projection operator with set $\Omega=\{\textit{\textbf{d}} ~|~\textit{\textbf{d}}^-\leqslant\textit{\textbf{d}}\leqslant\textit{\textbf{d}}^+\}\subset \mathbb{R}^{2n+2m}$. $\textit{\textbf{d}}=[\Theta; \iota]$, $\textit{\textbf{d}}^+=[\tilde{\zeta}^+(t); \omega \textit{\textbf{l}}_{\iota}]\in \mathbb{R}^{2n+2m}$,
$\textit{\textbf{d}}^-=[\tilde{\zeta}^-(t); -\omega \textit{\textbf{l}}_{\iota}]\in \mathbb{R}^{2n+2m}$,
$\Gamma=[\textit{\textbf{M}}, -\textit{\textbf{J}}^\text{T}(\Theta); \textit{\textbf{J}}(\Theta), \textbf{\text{0}}]\in \mathbb{R}^{(2n+2m)\times (2n+2m)}$,
$\textit{\textbf{q}}=[\textbf{\text{0}}; -\dot{\Upsilon}]\in \mathbb{R}^{n+m}$, and $\textit{\textbf{l}}_{\iota}=[\textbf{1}, \cdots, \textbf{1}]^\text{T}$.
Besides, $\textit{\textbf{d}}\in \mathbb{R}^m$ denotes the primal-dual decision vector, $\textit{\textbf{d}}^+\in \mathbb{R}^m$ and $\textit{\textbf{d}}^-\in \mathbb{R}^m$ are the upper and lower bounds, severally. $\omega$ is valued enormous ($\varpi:=10^{10}$) in the experiments.
We define $\varepsilon(\textit{\textbf{d}}):= \textit{\textbf{d}}-\Phi_{\Omega}\big(\textit{\textbf{d}}-(\Gamma \textit{\textbf{d}}+\textit{\textbf{q}})\big)$, and the iterative algorithm causes $\varepsilon(\textit{\textbf{d}})\rightarrow 0$. $\textit{\textbf{d}}^0\in \mathbb{R}^{2n+2m}$ is the original primal-dual decision variable vector, $k=0,1,2,3,\cdots$, if $\textit{\textbf{d}}^k\notin \Omega^*$, then
\begin{equation}
\textit{\textbf{d}}^{k+1}=\textit{\textbf{d}}^k-\frac{||\varepsilon(\textit{\textbf{d}}^k)||_2^2\sigma(\textit{\textbf{d}}^k)}
{||\sigma(\textit{\textbf{d}}^k)||_2^2}\label{eq36}
\end{equation}
where $\varepsilon(\textit{\textbf{d}}^k)=\textit{\textbf{d}}^k-\Phi_{\Omega}(\textit{\textbf{d}}^k-(\Gamma \textit{\textbf{d}}^k+\textit{\textbf{q}}))$ and $\sigma(\textit{\textbf{d}}^k)=(\Gamma+I)\varepsilon(\textit{\textbf{d}}^k)$. In numerical calculation, $\varepsilon(\textit{\textbf{d}}^k)=10^{-5}$.

\begin{lemma}
\cite{article.9} The sequence $\{\textit{\textbf{d}}^k\}, k=1, 2, \cdots$, which is generated by QP-solution (\ref{eq36}), satisfies
\begin{equation}
||\textit{\textbf{d}}^{k+1}-\textit{\textbf{d}}^{\star}||_2^2\leqslant||\textit{\textbf{d}}^k-\textit{\textbf{d}}^{\star}||_2^2-
\frac{||\varepsilon(\textit{\textbf{d}}^k)_2^4}{||\sigma(\textit{\textbf{d}}^k)||_2^2}.
\end{equation}
\end{lemma}

This lemma is applied to all $\textit{\textbf{d}}^{\star}\in\Omega^{\star}$. The solution vector $\textit{\textbf{d}}^{\star}$ comes from $\{\textit{\textbf{d}}^k\}$, and its first $2n$ elements consist of the optimal solution to MGDDF (\ref{eq32})-(\ref{eq34}), $\Theta^{\star}\in \mathbb{R}^n$. Definitely, the first $n$ elements constitute the optimal solutions of the left-arm joints, while the second $n$ elements constitute right-arm joints' optimal solution.

\begin{corollary}
For any positive constant $\Xi>0$, error vector $\epsilon\in\mathbb{R}^n$ of a dynamic system $\dot{\epsilon}=-\Xi\epsilon$, starting from any initial state $\epsilon(0)$, converges to zero with exponential convergence speed, where $\dot{\epsilon}$ is the first-order derivative of $\epsilon$ with respect to time $t$.
\end{corollary}

\begin{corollary}
For any positive constant $\Xi>0$ and disturbance $\Delta\Xi>-\Xi$, a dynamic system with disturbance parameter $\dot{\epsilon}=-(\Xi+\Delta\Xi)\epsilon$ is robust for the uncertain-parameter situation, where $\dot{\epsilon}$ is the first-order derivative of $\epsilon$ with respect to time $t$.
\end{corollary}

According to dual theory and Lagrange theory \cite{NN29,NN30}, the proposed MGDDF (\ref{eq32})-(\ref{eq34}) can also be solved by a neural network solver. To find a primal dual equilibrium vector $\mathcal{K}\in\mathbb{R}^{2m}$, MGDDF (\ref{eq32})-(\ref{eq34}) should be converted into the following linear variational inequalities, i.e.,
\begin{equation}
\begin{bmatrix}
\dot{\Theta}^*\\
\mathcal{K}^*
\end{bmatrix}\in\Omega
:=\bigg\{
\begin{bmatrix}
\dot{\Theta}^*\\
\mathcal{K}^*
\end{bmatrix}
~\bigg|~
\begin{bmatrix}
\tilde{\zeta}_{\emph{i}}^-(t)\\
-z\textit{\textbf{l}}_{\mathcal{K}}
\end{bmatrix}
\leq
\begin{bmatrix}
\dot{\Theta}^*\\
\mathcal{K}^*
\end{bmatrix}
\leq
\begin{bmatrix}
\tilde{\zeta}_{\emph{i}}^+(t)\\
+z\textit{\textbf{l}}_{\mathcal{K}}
\end{bmatrix}\bigg\}\subset\mathbb{R}^{2n+2m}
\end{equation}
where $\textit{\textbf{l}}_{\mathcal{K}}=[\textbf{1}, \cdots, \textbf{1}]^\text{T}$ and $z$ is set to approximate $m$-dimensional $+\infty$. Such that
\begin{equation}
\bigg(
\begin{bmatrix}
\dot{\Theta}\\
\mathcal{K}
\end{bmatrix}-
\begin{bmatrix}
\dot{\Theta}^*\\
\mathcal{K}^*
\end{bmatrix}
\bigg)^{\text{T}}
\bigg(\begin{bmatrix}
\mathcal{I} &-\textit{\textbf{J}}^{\text{T}}\\
\textit{\textbf{J}} &\textbf{0}
\end{bmatrix}
\begin{bmatrix}
\dot{\Theta}^*\\
\mathcal{K}^*
\end{bmatrix}+
\begin{bmatrix}
\textit{\textbf{c}}^\text{T}\Theta\\
-\dot{\Upsilon}-\wp(\Upsilon-\Im(\Theta))
\end{bmatrix}
\bigg)\geq0, \forall \begin{bmatrix}
\dot{\Theta}\\
\mathcal{K}
\end{bmatrix}\in\Omega
\end{equation}
where $\mathcal{I}\in\mathbb{R}^{2n\times2n}$. According to Ref. \cite{myarticle.21}, the above relationship is equivalent to a piecewise-linear-equation system satisfy
\begin{equation}
\mathcal{P}_{\Omega}\bigg(
\begin{bmatrix}
\dot{\Theta}\\
\mathcal{K}
\end{bmatrix}-
\big(\begin{bmatrix}
\mathcal{I} &-\textit{\textbf{J}}^{\text{T}}\\
\textit{\textbf{J}} &\textbf{0}
\end{bmatrix}
\begin{bmatrix}
\dot{\Theta}\\
\mathcal{K}
\end{bmatrix}+
\begin{bmatrix}
\textit{\textbf{c}}^\text{T}\Theta\\
-\dot{\Upsilon}-\wp(\Upsilon-\Im(\Theta))
\end{bmatrix}\big)\bigg)-
\begin{bmatrix}
\dot{\Theta}\\
\mathcal{K}
\end{bmatrix}=0
\label{equ.nn-1}
\end{equation}
where $\mathcal{P}_{\Omega}(\cdot):\mathbb{R}^{2n+2m}\rightarrow\Omega$ is the projection operator.

With projective theory, system (\ref{equ.nn-1}) can be computed by the recurrent neural network below
\begin{equation}
\begin{bmatrix}
\ddot{\Theta}\\
\dot{\mathcal{K}}
\end{bmatrix}
=\mathcal{H}\bigg(\mathcal{E}+
\begin{bmatrix}
\mathcal{I} &\textit{\textbf{J}}\\
-\textit{\textbf{J}}^{\text{T}} &\textbf{0}
\end{bmatrix}\bigg)
\cdot
\bigg\{
\mathcal{P}_{\Omega}\bigg(
\begin{bmatrix}
\dot{\Theta}\\
\mathcal{K}
\end{bmatrix}-
\big(\begin{bmatrix}
\mathcal{I} &-\textit{\textbf{J}}^{\text{T}}\\
\textit{\textbf{J}} &\textbf{0}
\end{bmatrix}
\begin{bmatrix}
\dot{\Theta}\\
\mathcal{K}
\end{bmatrix}+
\begin{bmatrix}
\textit{\textbf{c}}^\text{T}\Theta\\
-\dot{\Upsilon}-\wp(\Upsilon-\Im(\Theta))
\end{bmatrix}\big)\bigg)-
\begin{bmatrix}
\dot{\Theta}\\
\mathcal{K}
\end{bmatrix}
\bigg\}
\label{equ.nn-2}
\end{equation}
where $\mathcal{H}>0$ is designed to scale convergent rate and $\mathcal{E}\in\mathbb{R}^{(2n+2m)\times(2n+2m)}$ denotes identity matrix.

\begin{theorem}
State vector $[\dot{\Theta}, \mathcal{K}]^{\text{T}}$ of (\ref{equ.nn-2}), starting from any initial state, can globally convergent to an equilibrium point $[\dot{\Theta}^*, \mathcal{K}^*]^{\text{T}}$, of which the first $2n$ elements constitute optimal solution $\dot{\Theta}^*$ to MGDDF (\ref{eq32})-(\ref{eq34}). Moreover, with constant $\mathcal{N}>0$ such that
\begin{equation}
\bigg\|~
\begin{bmatrix}
\ddot{\Theta}\\
\dot{\mathcal{K}}
\end{bmatrix}
-
\mathcal{P}_{\Omega}\bigg(
\begin{bmatrix}
\dot{\Theta}\\
\mathcal{K}
\end{bmatrix}-
\big(\begin{bmatrix}
\mathcal{I} &-\textit{\textbf{J}}^{\text{T}}\\
\textit{\textbf{J}} &\textbf{0}
\end{bmatrix}
\begin{bmatrix}
\dot{\Theta}\\
\mathcal{K}
\end{bmatrix}+
\begin{bmatrix}
\textit{\textbf{c}}^\text{T}\Theta\\
-\dot{\Upsilon}-\wp(\Upsilon-\Im(\Theta))
\end{bmatrix}\big)\bigg)
~\bigg\|_2^2
\geq
\mathcal{N}~\bigg\|~\begin{bmatrix}
\ddot{\Theta}\\
\dot{\mathcal{K}}
\end{bmatrix}
-
\begin{bmatrix}
\ddot{\Theta}^*\\
\dot{\mathcal{K}}^*
\end{bmatrix}~
\bigg\|_2^2,
\end{equation}
the exponential stability for computing MGDDF (\ref{eq32})-(\ref{eq34}) can be guaranteed.
\end{theorem}

\begin{proof}
Firstly, define the following Lyapunov function candidate, i.e.,
\begin{equation}
\mathcal{V}\bigg(
\begin{bmatrix}
\ddot{\Theta}\\
\dot{\mathcal{K}}
\end{bmatrix}
\bigg)=
\bigg\|~\begin{bmatrix}
\ddot{\Theta}\\
\dot{\mathcal{K}}
\end{bmatrix}
-
\begin{bmatrix}
\ddot{\Theta}^*\\
\dot{\mathcal{K}}^*
\end{bmatrix}
~\bigg\|_2^2
\geq0.
\end{equation}

Secondly, the time derivative of $\mathcal{V}(\cdot)$ along (\ref{equ.nn-2}) is
\begin{equation}
\frac{\text{d}\mathcal{V}\bigg(
\begin{bmatrix}
\ddot{\Theta}\\
\dot{\mathcal{K}}
\end{bmatrix}
\bigg)}{\text{d}t}
\end{equation}
\begin{equation}
=
\mathcal{H}
\bigg(
\begin{bmatrix}
\ddot{\Theta}\\
\dot{\mathcal{K}}
\end{bmatrix}
-
\begin{bmatrix}
\ddot{\Theta}^*\\
\dot{\mathcal{K}}^*
\end{bmatrix}
\bigg)^{\text{T}}
\bigg(\mathcal{E}+
\begin{bmatrix}
\mathcal{I} &\textit{\textbf{J}}\\
-\textit{\textbf{J}}^{\text{T}} &\textbf{0}
\end{bmatrix}\bigg)
\times
\bigg\{
\mathcal{P}_{\Omega}\bigg(
\begin{bmatrix}
\dot{\Theta}\\
\mathcal{K}
\end{bmatrix}-
\big(\begin{bmatrix}
\mathcal{I} &-\textit{\textbf{J}}^{\text{T}}\\
\textit{\textbf{J}} &\textbf{0}
\end{bmatrix}
\begin{bmatrix}
\dot{\Theta}\\
\mathcal{K}
\end{bmatrix}+
\begin{bmatrix}
\textit{\textbf{c}}^\text{T}\Theta\\
-\dot{\Upsilon}-\wp(\Upsilon-\Im(\Theta))
\end{bmatrix}\big)\bigg)-
\begin{bmatrix}
\dot{\Theta}\\
\mathcal{K}
\end{bmatrix}
\bigg\}
\end{equation}
\begin{equation}
\leq-\mathcal{H}\bigg\|~
\begin{bmatrix}
\dot{\Theta}\\
\mathcal{K}
\end{bmatrix}
-
\mathcal{P}_{\Omega}\bigg(
\begin{bmatrix}
\dot{\Theta}\\
\mathcal{K}
\end{bmatrix}-
\big(\begin{bmatrix}
\mathcal{I} &-\textit{\textbf{J}}^{\text{T}}\\
\textit{\textbf{J}} &\textbf{0}
\end{bmatrix}
\begin{bmatrix}
\dot{\Theta}\\
\mathcal{K}
\end{bmatrix}+
\begin{bmatrix}
\textit{\textbf{c}}^\text{T}\Theta\\
-\dot{\Upsilon}-\wp(\Upsilon-\Im(\Theta))
\end{bmatrix}\big)\bigg)
~\bigg\|^2_2
-
\mathcal{H}\bigg\|~
\bigg(
\begin{bmatrix}
\ddot{\Theta}\\
\dot{\mathcal{K}}
\end{bmatrix}
-
\begin{bmatrix}
\ddot{\Theta}^*\\
\dot{\mathcal{K}}^*
\end{bmatrix}
\bigg)
~\bigg\|_2^2
\leq0.
\end{equation}

On the basis of Lyapunov theory, with positive-defined $\mathcal{V}$ and negative-dfined $\dot{\mathcal{V}}$, state vector $[\dot{\Theta}, \mathcal{K}]^{\text{T}}$ of (\ref{equ.nn-2}) is stable and converges to equilibrium point $[\dot{\Theta}^*, \mathcal{K}^*]^{\text{T}}$ globally, in view that $\dot{\mathcal{V}}=0$ when $[\dot{\Theta}, \mathcal{K}]^{\text{T}}=[\dot{\Theta}^*, \mathcal{K}^*]^{\text{T}}=0$. It follows that the first $2n$ elements of $[\dot{\Theta}^*, \mathcal{K}^*]^{\text{T}}$ constitute the optimal solution to MGDDF (\ref{eq32})-(\ref{eq34}). Also note that
\begin{equation}
\frac{\text{d}\mathcal{V}\bigg(
\begin{bmatrix}
\ddot{\Theta}\\
\dot{\mathcal{K}}
\end{bmatrix}
\bigg)}{\text{d}t}
\leq
-
\mathcal{H}
\bigg(
\begin{bmatrix}
\ddot{\Theta}\\
\dot{\mathcal{K}}
\end{bmatrix}
-
\begin{bmatrix}
\ddot{\Theta}^*\\
\dot{\mathcal{K}}^*
\end{bmatrix}
\bigg)^{\text{T}}
\bigg(
\mathcal{N}\mathcal{E}+
\begin{bmatrix}
\mathcal{I} &-\textit{\textbf{J}}^{\text{T}}\\
\textit{\textbf{J}} &\textbf{0}
\end{bmatrix}
\bigg)\cdot
\bigg(
\begin{bmatrix}
\ddot{\Theta}\\
\dot{\mathcal{K}}
\end{bmatrix}
-
\begin{bmatrix}
\ddot{\Theta}^*\\
\dot{\mathcal{K}}^*
\end{bmatrix}
\bigg)
\leq-\mathcal{S}\mathcal{V}\bigg(
\begin{bmatrix}
\ddot{\Theta}\\
\dot{\mathcal{K}}
\end{bmatrix}
\bigg)
\end{equation}
where convergence rate $\mathcal{S}=\mathcal{N}\mathcal{H}>0$. Therefore, with $\forall t\geq t_0$, we have $\mathcal{V}([\dot{\Theta}, \mathcal{K}]^{\text{T}})=O(e^{-\mathcal{S}(t-t_0)})$ such that
\begin{equation}
\bigg\|~
\begin{bmatrix}
\ddot{\Theta}\\
\dot{\mathcal{K}}
\end{bmatrix}
-
\begin{bmatrix}
\ddot{\Theta}^*\\
\dot{\mathcal{K}}^*
\end{bmatrix}
~\bigg\|_2^2
=
O(e^{-\mathcal{S}(t-t_0)/2}).
\label{eqn.nn-3}
\end{equation}

In reference to (\ref{eqn.nn-3}), we can draw the conclusion that the exponential computation speed of MGDDF (\ref{eq32})-(\ref{eq34}) can be achieved. The proof is thus completed.
\end{proof}

%%%%%% Section V: Simulations and Experiments
\section{Simulations and Experiments}

%%%%%% FIG-4
\begin{figure*}
  \centering
  \includegraphics[width=.75\textwidth]{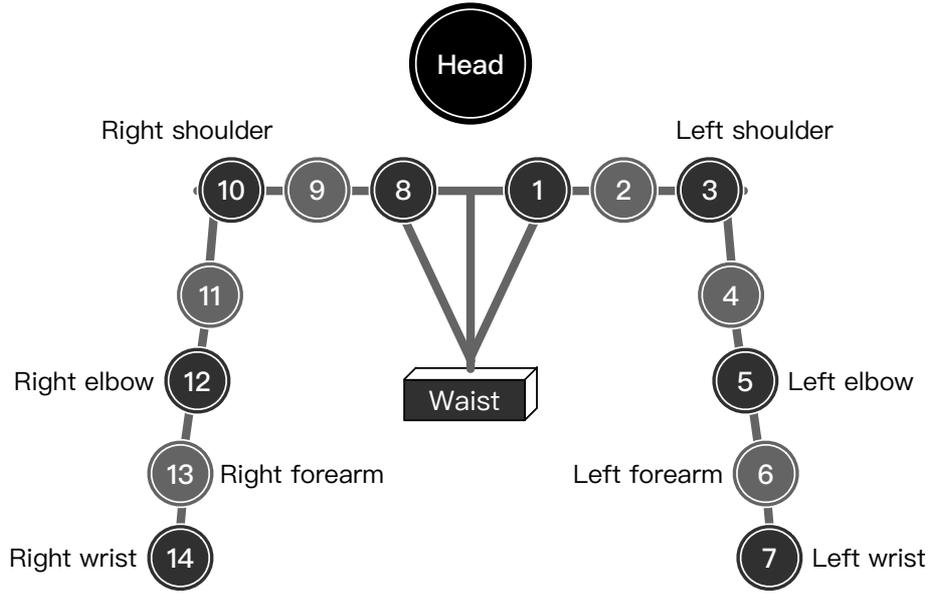}\\
  \caption{The skeleton structure of dual arms of the humanoid robot used in simulations.}\label{fig.structure}
\end{figure*}
%%%%%% FIG-4

In this section, the practicability of the proposed MGDDF scheme is verified by simulative experiments. The comparison between MGDDF and the traditional GDDF scheme is also provided. Specifically, the experiments are based on the dual arms of a humanoid robot, which contain 14 DOF (7 of each arm), and the skeleton structure is shown in Fig. \ref{fig.structure}. The execution task lasts for $18s$. In addition, the upper and lower bounds of the joint limit $6$, $7$, $13$ and $14$ are set as the same value, respectively.

\subsection{Adjustments of the GDDF Scheme}

Set the parameters of the QP problem (\ref{eq32})-(\ref{eq34}) as follows: $N=1$, $\varrho=2$, $\kappa=0.85$, $\Theta_{gL}^-=[0, -54\pi/180,$ $-10.5\pi/180, 0, -131\pi/180, \pi/3, 55\pi/180]^\text{T}$ (rad),
$\Theta_{gL}^+=[9\pi/180, 18\pi/180, 22.5\pi/180, \pi/2, 0, \pi/3,$ $55\pi/180]^\text{T}$ (rad), $\Theta_{gR}^-=[-9\pi/180, -18\pi/180, -22.5\pi$ $/180, \pi/2, -131\pi/180, \pi/3, -25\pi/36]^\text{T}$ (rad), and $\Theta_{gR}^+=$ $[0, 54\pi/180,$ $10.5\pi/180, \pi, 0, \pi/3, -25\pi/36]^\text{T}$ (rad).

At first, to show the adjustment to the traditional GDDF scheme, the inequality (\ref{eq34}) of the QP problem (\ref{eq32})-(\ref{eq34}) is not applied directly. For the GDDF scheme, when the margin is not considered, the upper limit and lower limits would overlap and the result is shown in Fig. \ref{fig.compare adjustment}. Specifically, the black line stands for the joint limit and the green line stands for the limit with margins. From the figure we can see that the movement of the joint is constrained by the limit with margins.
Because of the margins, $\Theta_{g}^-$ and $\Theta_{g}^+$ of joint $7$ is not the same, so we propose that when the limit with margins reaches the goal value, they would hold that value. After adjustment, the joint would not exceed the margins. Fig. \ref{fig.compare adjustment} (b) gives the result of this modification.

%%%%%% FIG-5
\begin{figure}[tbp]
\centering
\subfigure[]{\includegraphics[width=0.49\columnwidth]{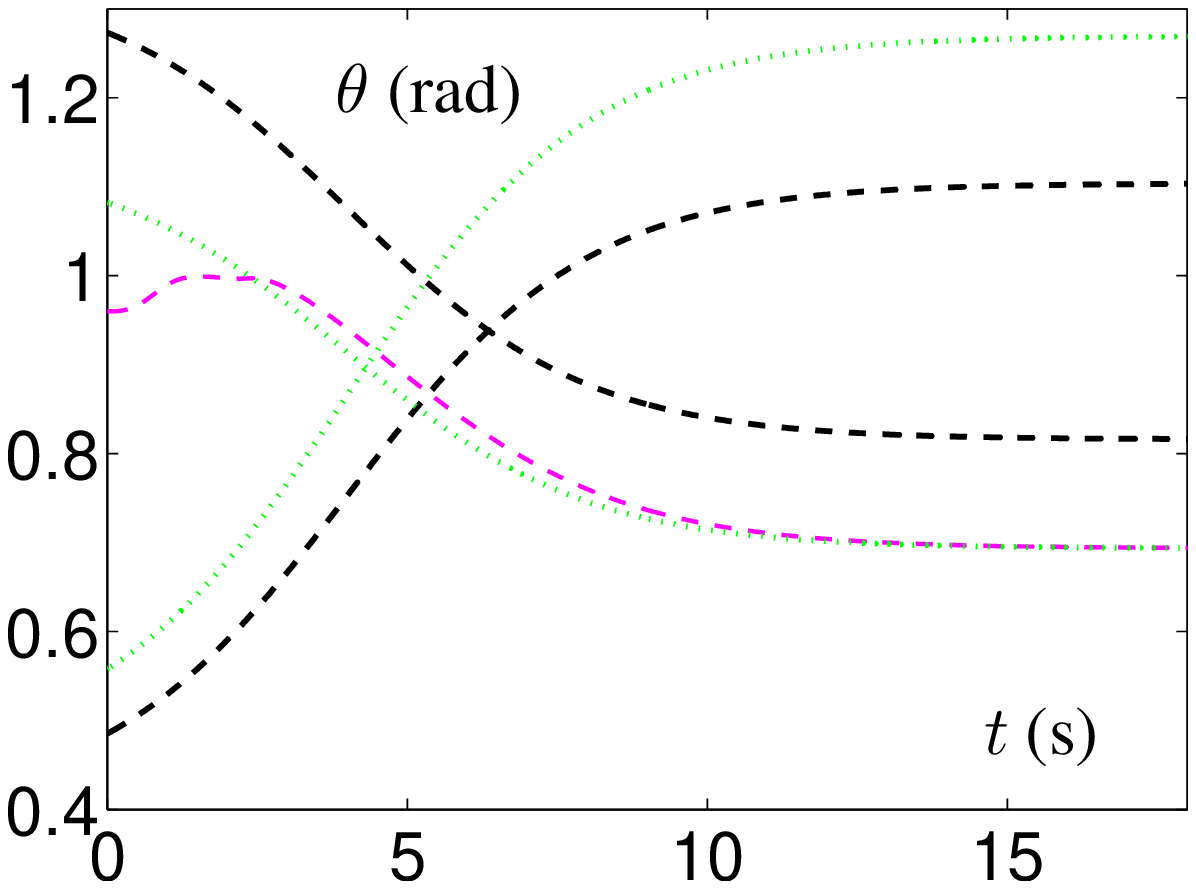}}
\subfigure[]{\includegraphics[width=0.49\columnwidth]{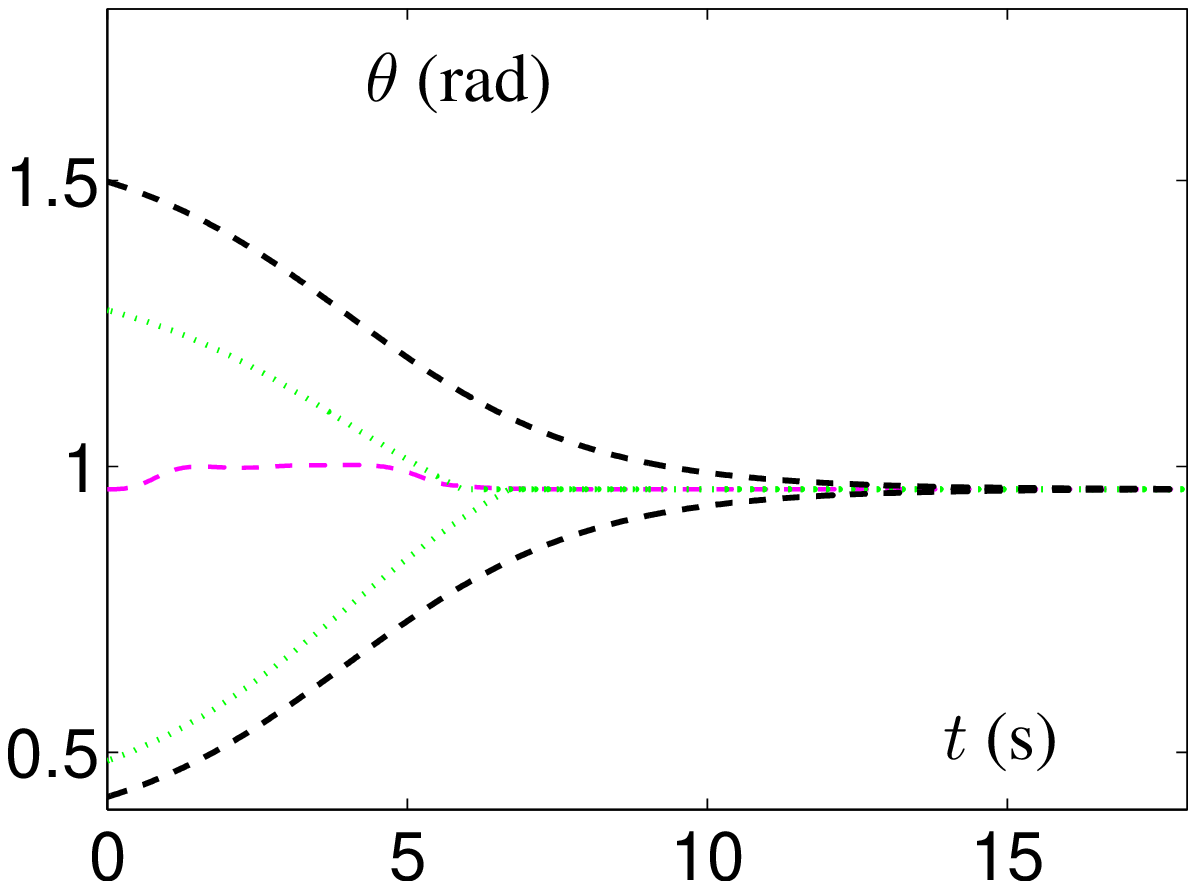}}
\caption{Comparisons between GDDF scheme with non-adjustment-margin and adjustment-margin. (a) Without adjustment. (b) With adjustment. }\label{fig.compare adjustment}
\end{figure}
%%%%%% FIG-4

We use the following equations to simplify, i.e.,
\begin{eqnarray}
\Theta_{set}^+(t)=\Theta^++\frac{\Delta\Theta^+}{1+e^{-(t-\tau)/\varrho}}\\
\Theta_{set}^-(t)=\Theta^-+\frac{\Delta\Theta^-}{1+e^{-(t-\tau)/\varrho}}.
\end{eqnarray}

For joint $6$, $7$ and $13$, the upper limit and lower limit both $>0$, the following equations are introduced, i.e.,
$$ \kappa\tilde{\Theta}^+ =\left\{
\begin{aligned}
&\kappa\Theta_{set}^+(t), && \kappa\tilde{\Theta}^+> \Theta_{g}^+\\
&\Theta_{g}^+, && \kappa\tilde{\Theta}^+ \leqslant \Theta_{g}^+
\end{aligned}
\right.
$$

$$ \kappa\tilde{\Theta}^- =\left\{
\begin{aligned}
&(2-\kappa)\Theta_{set}^-(t), && \kappa\tilde{\Theta}^+< \Theta_{g}^-\\
&\Theta_{g}^-, && \kappa\tilde{\Theta}^-\geqslant \Theta_{g}^-
\end{aligned}
\right.
$$
and function (\ref{fun18}) should be rewritten as
$$ \tilde{\Theta}^+ =\left\{
\begin{aligned}
&\Theta_{set}^+(t), && \tilde{\Theta}^+> \Theta_{g}^+/\kappa\\
&\Theta_{g}^+/\kappa, && \tilde{\Theta}^+ \leqslant \Theta_{g}^+/\kappa
\end{aligned}
\right.
$$

$$ \tilde{\Theta}^- =\left\{
\begin{aligned}
&\Theta_{set}^-(t), && \tilde{\Theta}^-< \Theta_{g}^-/(2-\kappa)\\
&\Theta_{g}^-/(2-\kappa), && \tilde{\Theta}^- \geqslant \Theta_{g}^-/(2-\kappa).
\end{aligned}
\right.
$$

%%%%%% FIG-6
\begin{figure}[tbp]
\centering
\subfigure[]{\includegraphics[width=0.245\columnwidth]{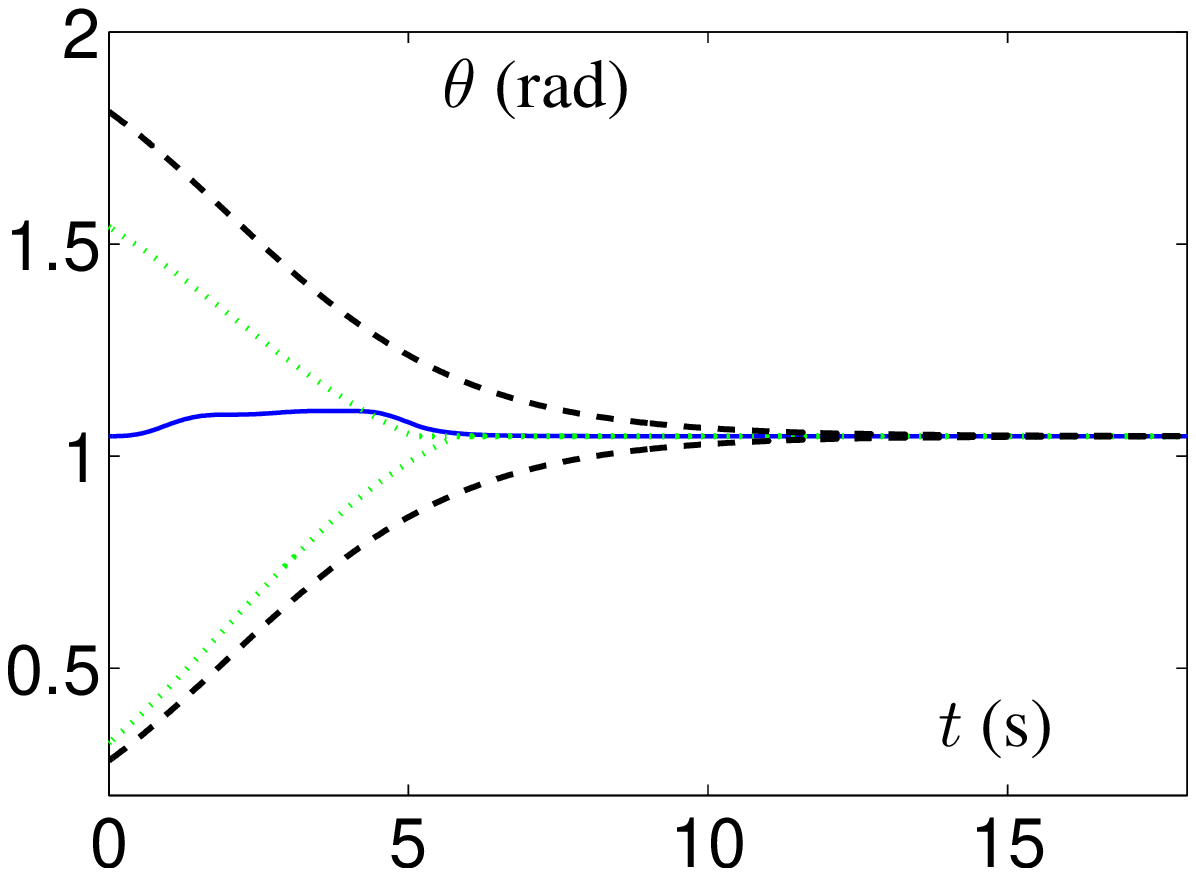}}
\subfigure[]{\includegraphics[width=0.245\columnwidth]{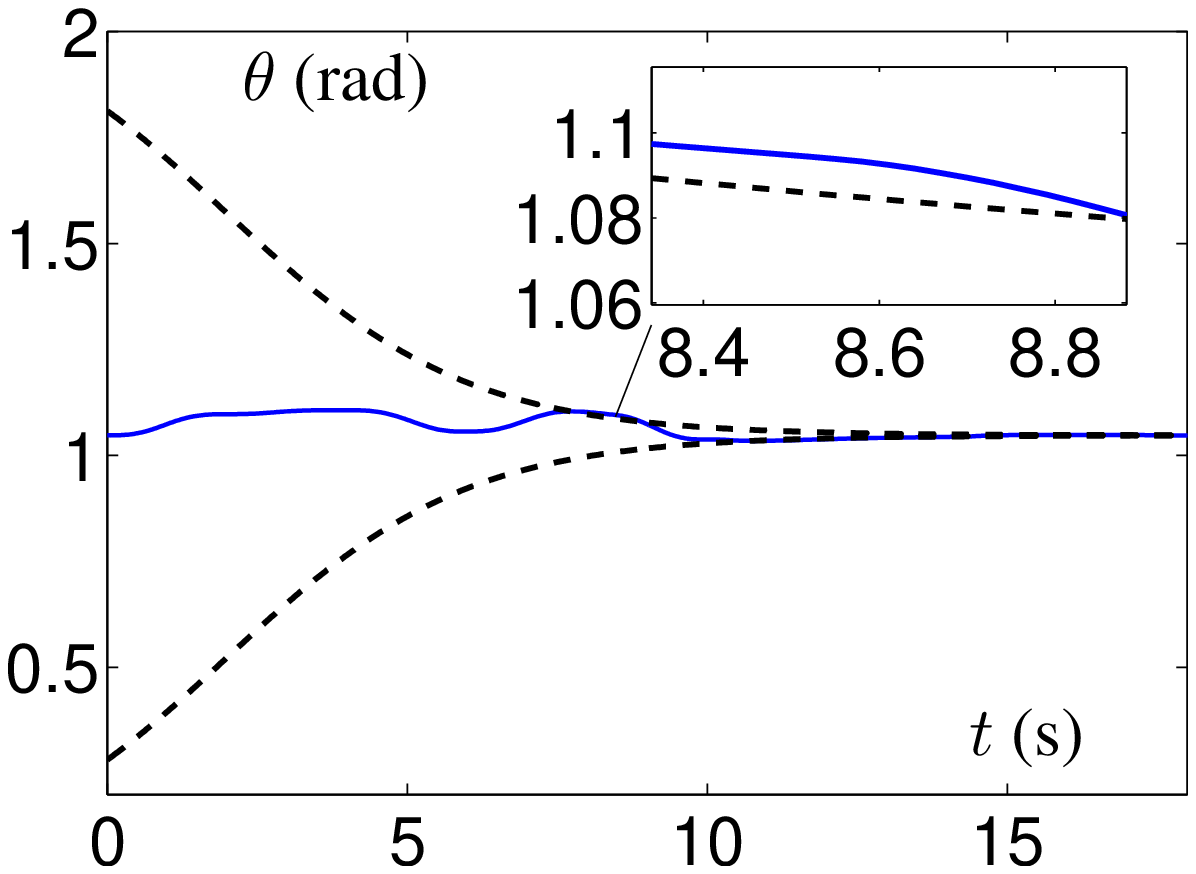}}
\subfigure[]{\includegraphics[width=0.245\columnwidth]{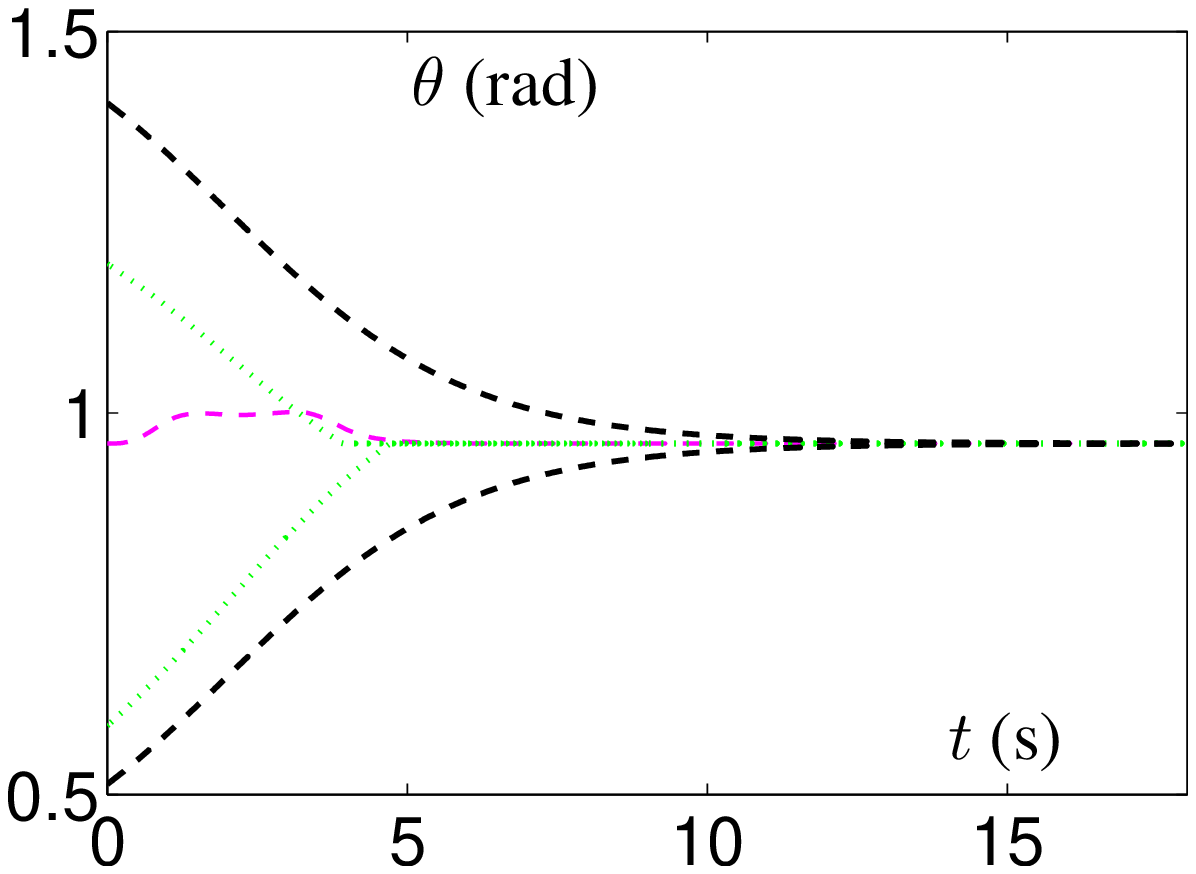}}
\subfigure[]{\includegraphics[width=0.245\columnwidth]{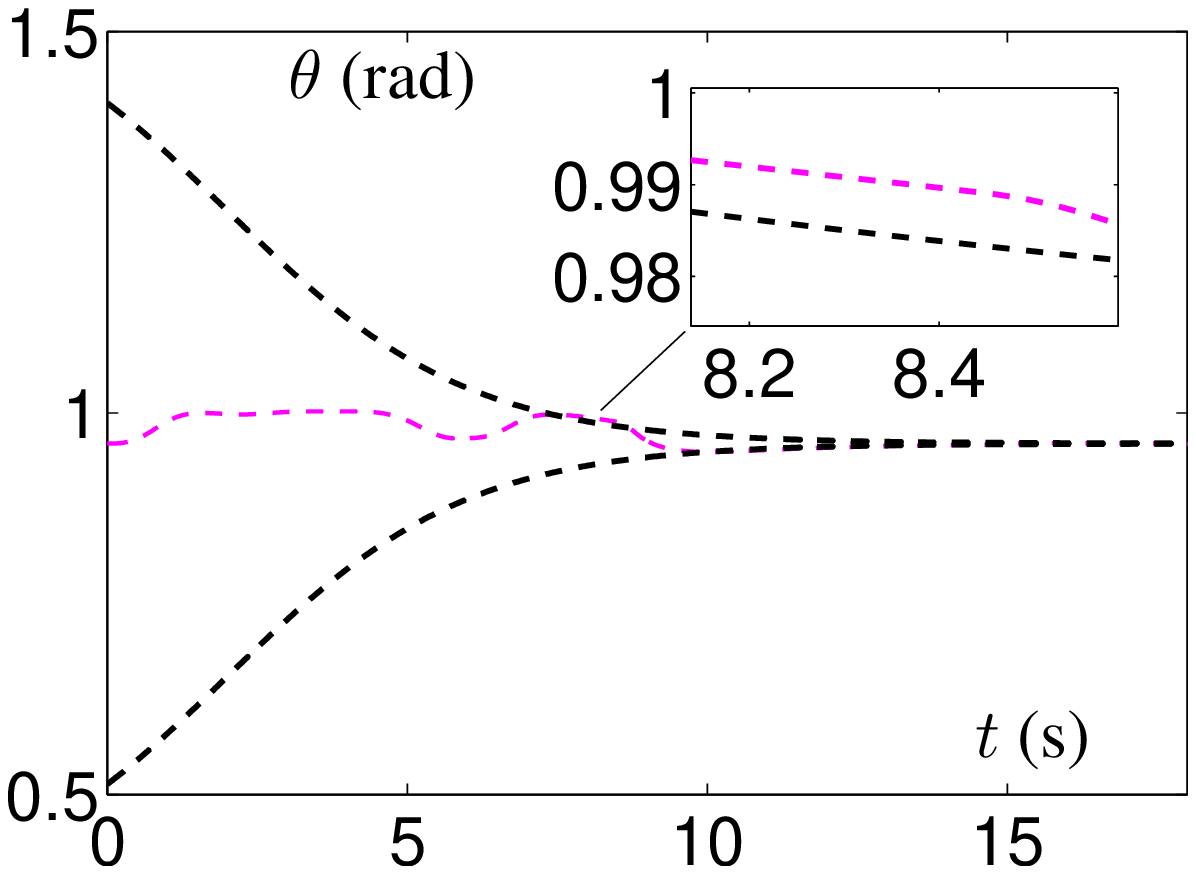}}
\caption{Comparisons between GDDF and MGDDF when $\varrho=2$, $N=2$. (a) Limit of joint $6$ with margin. (b) Limit of joint $6$ without margin. (c) Limit of joint $7$ with margin. (d) Limit of joint $7$ without margin.}\label{fig.c2n2}
\end{figure}
%%%%%% FIG-6

%%%%%% FIG-7
\begin{figure}[tbp]
\centering
\subfigure[]{\includegraphics[width=0.325\columnwidth]{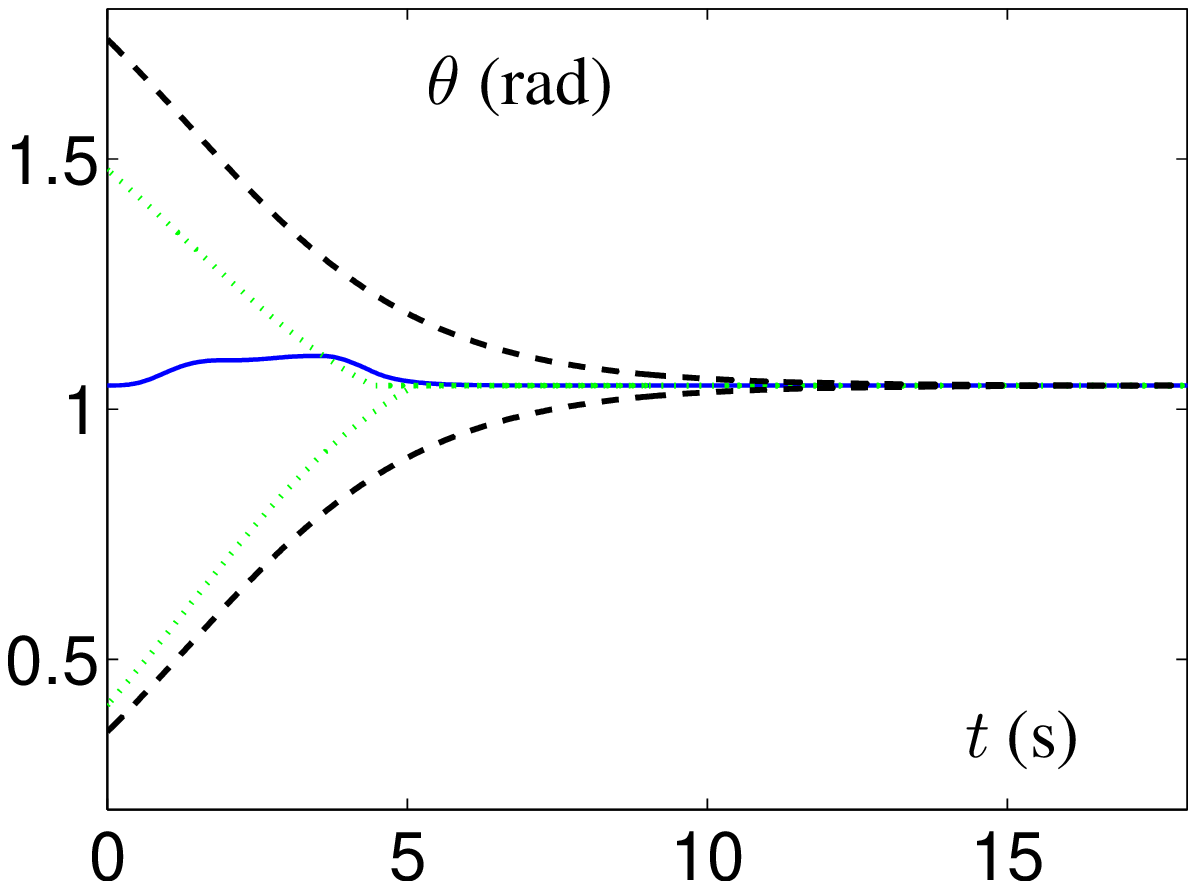}}
\subfigure[]{\includegraphics[width=0.325\columnwidth]{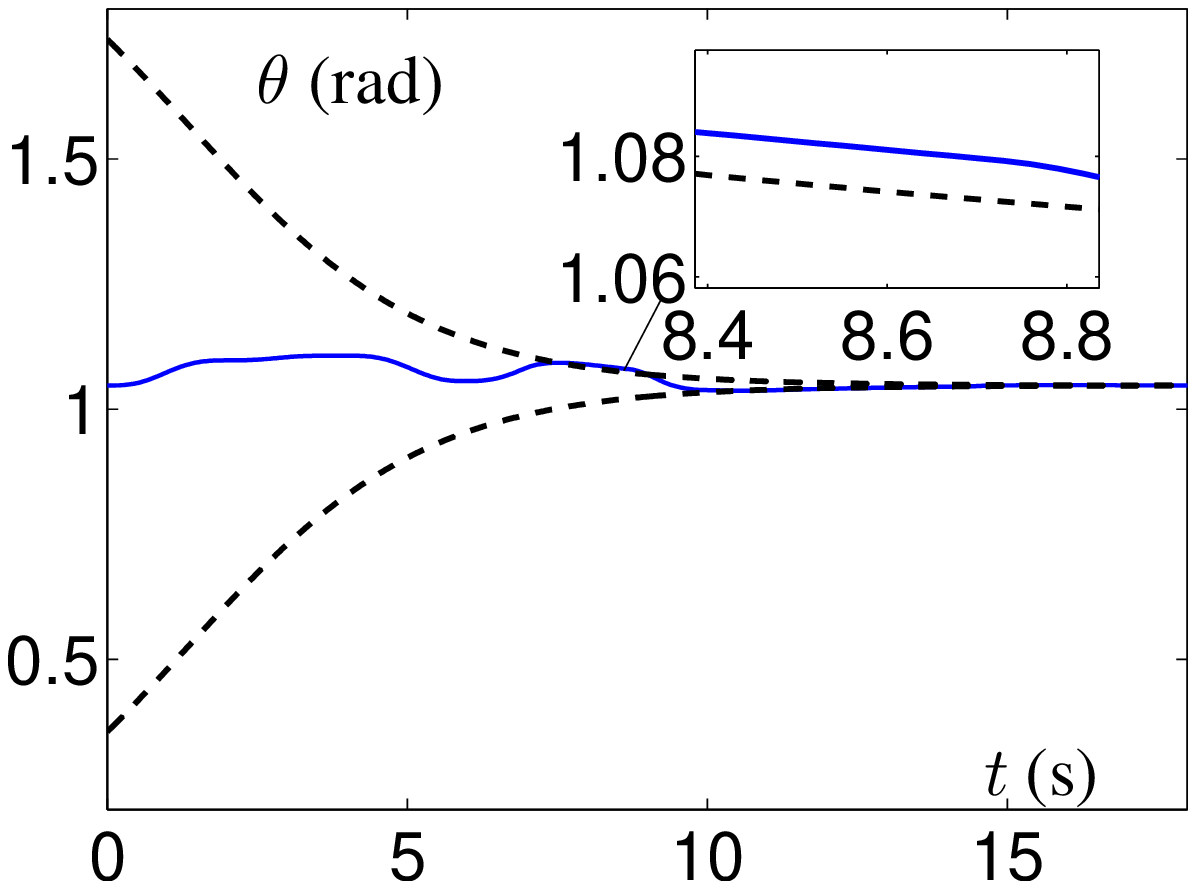}}
\subfigure[]{\includegraphics[width=0.34\columnwidth]{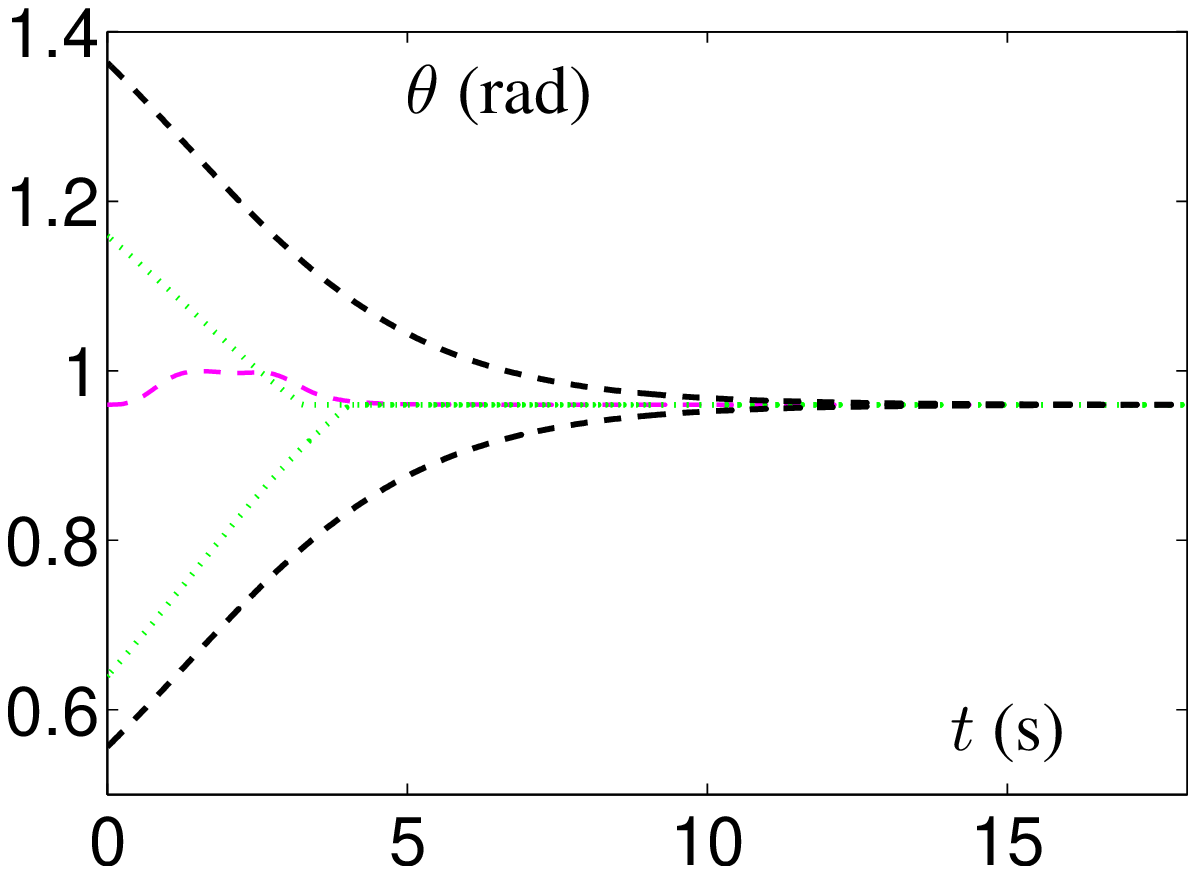}}
\subfigure[]{\includegraphics[width=0.34\columnwidth]{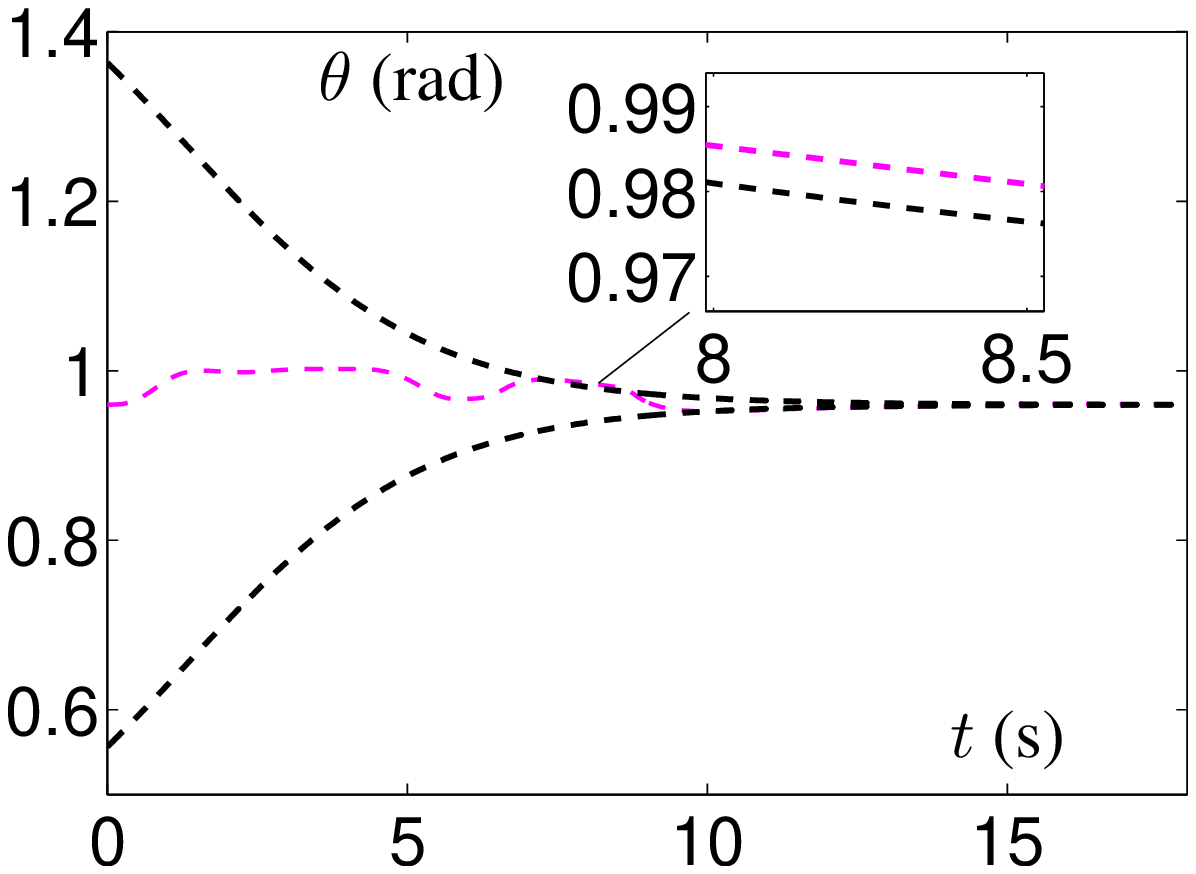}}
\subfigure[]{\includegraphics[width=0.325\columnwidth]{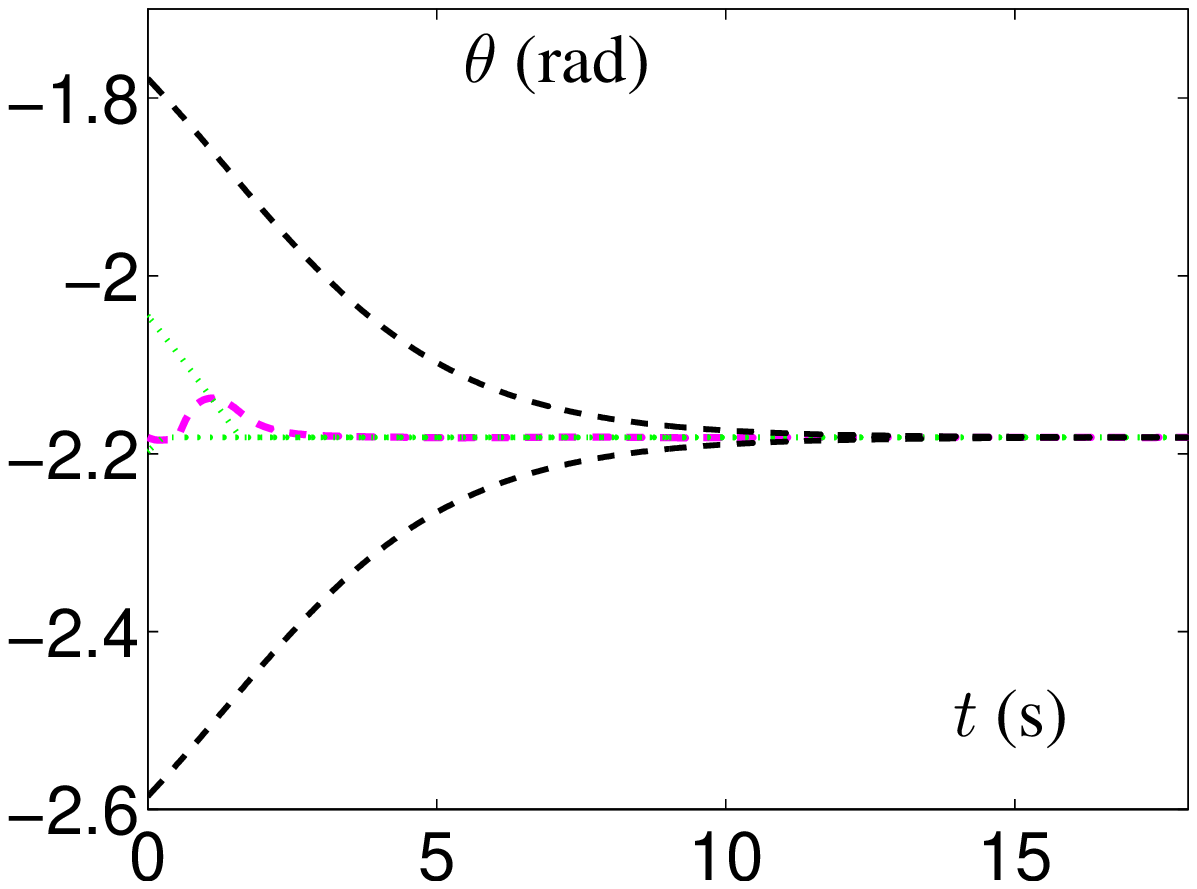}}
\subfigure[]{\includegraphics[width=0.325\columnwidth]{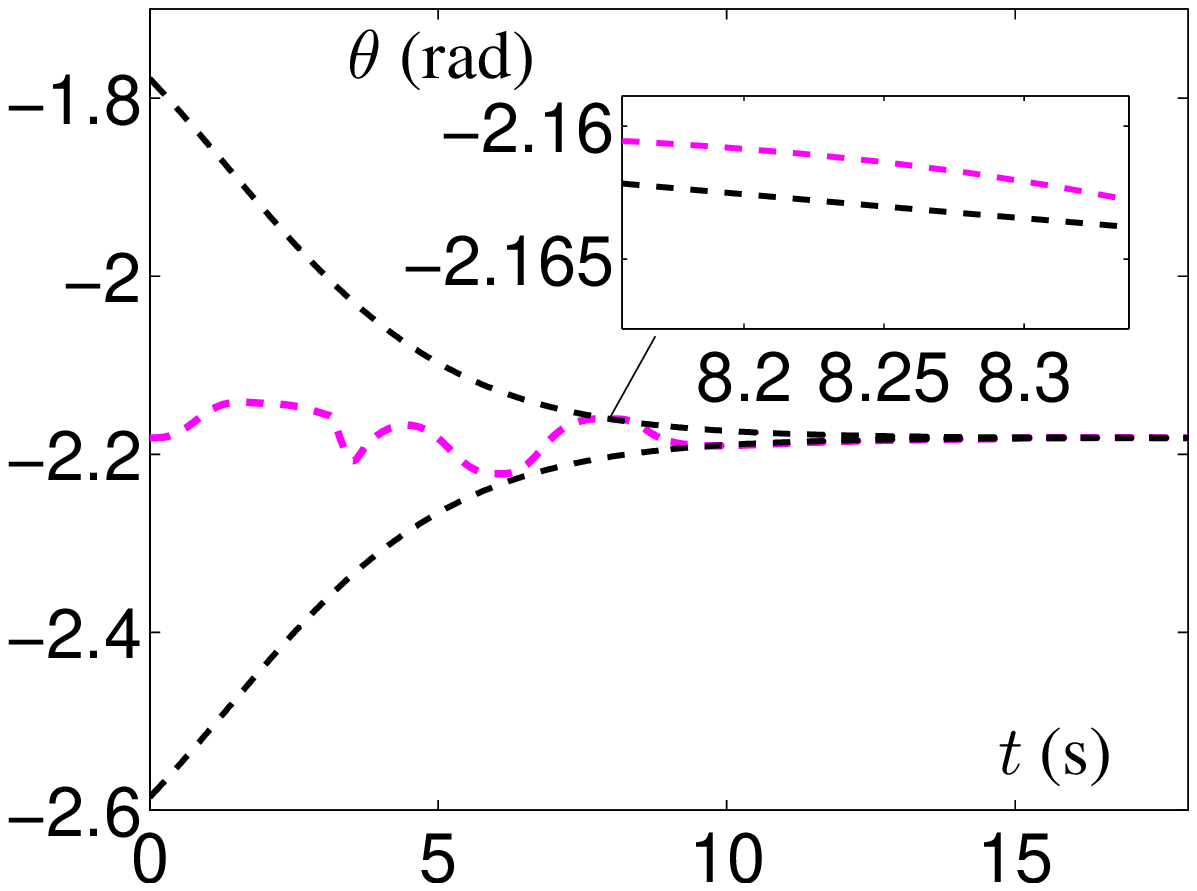}}
\caption{Comparisons between GDDF and MGDDF when $\varrho=2$, $N=3$. (a) Limit of joint $7$ with margin. (b) Limit of joint $6$ without margin. (c) Limit of joint $7$ with margin. (d) Limit of joint $7$ without margin. (e) Limit of Joint $14$ with margin. (d) Limit of joint $14$ without margin.} \label{fig.c2n3}
\end{figure}
%%%%%% FIG-7

%%%%%% FIG-8
\begin{figure}[tbp]
\centering
\subfigure[]{\includegraphics[width=0.325\columnwidth]{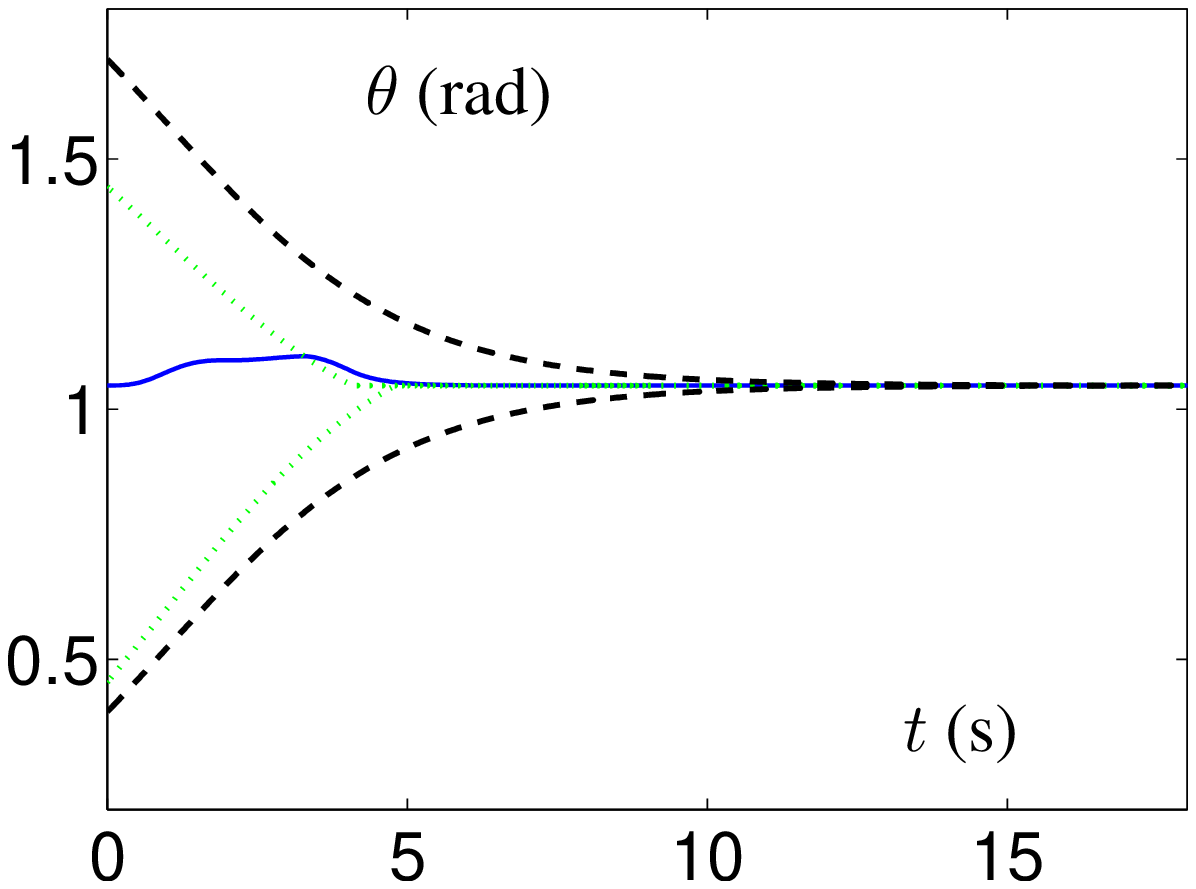}}
\subfigure[]{\includegraphics[width=0.325\columnwidth]{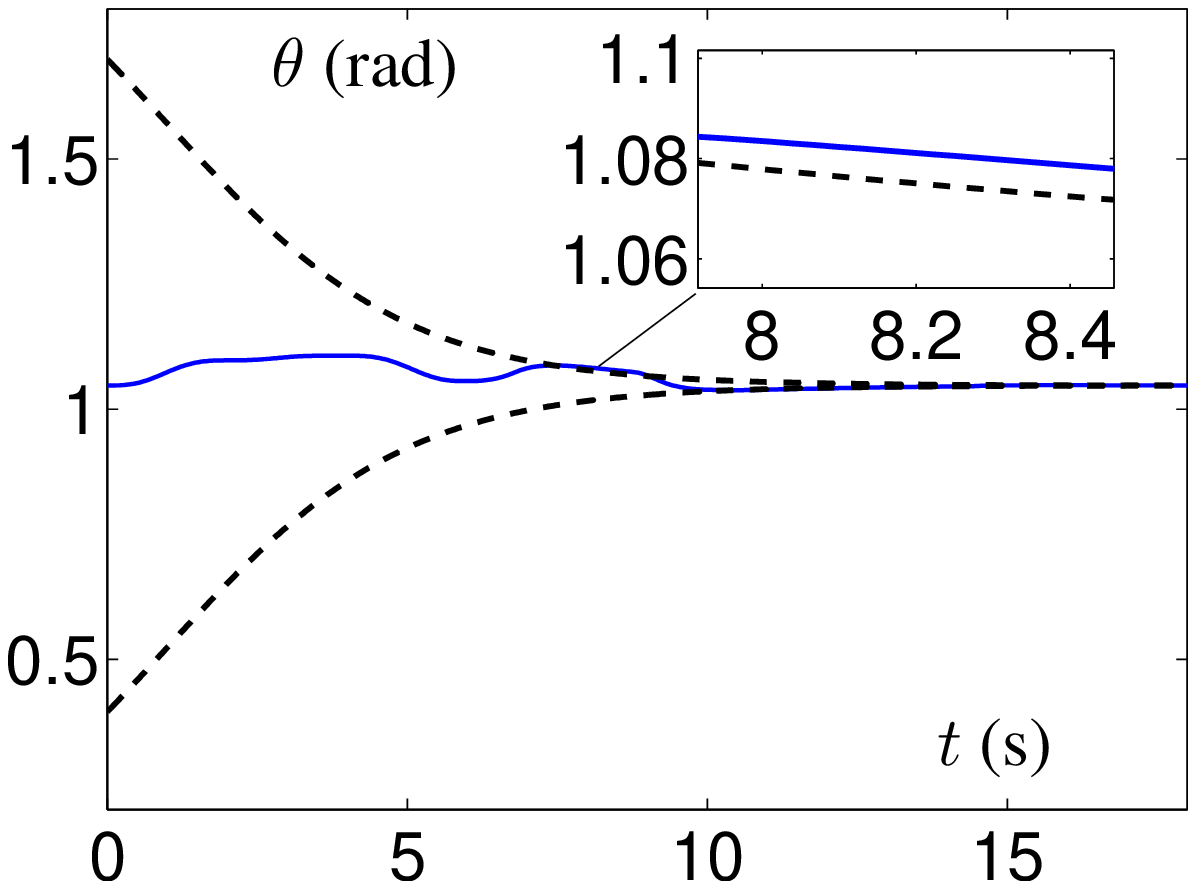}}
\subfigure[]{\includegraphics[width=0.34\columnwidth]{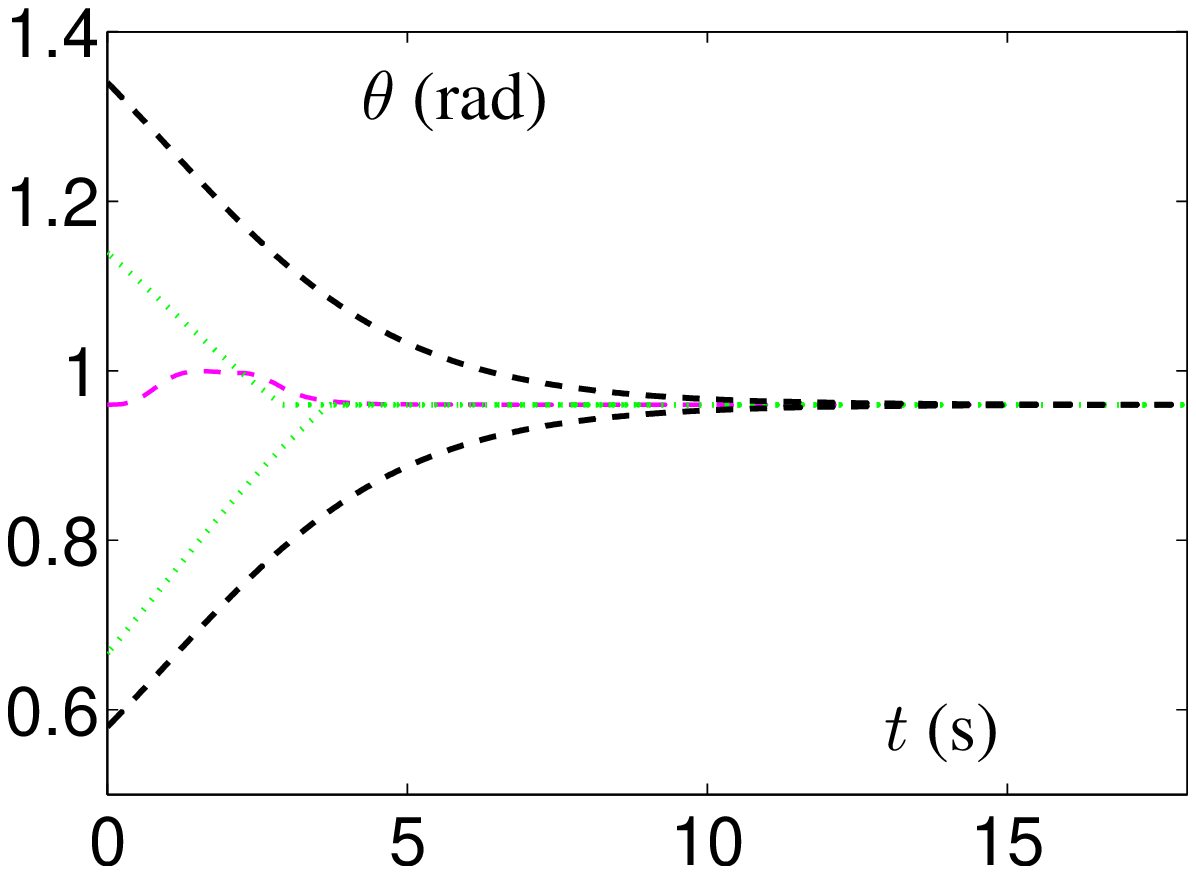}}
\subfigure[]{\includegraphics[width=0.33\columnwidth]{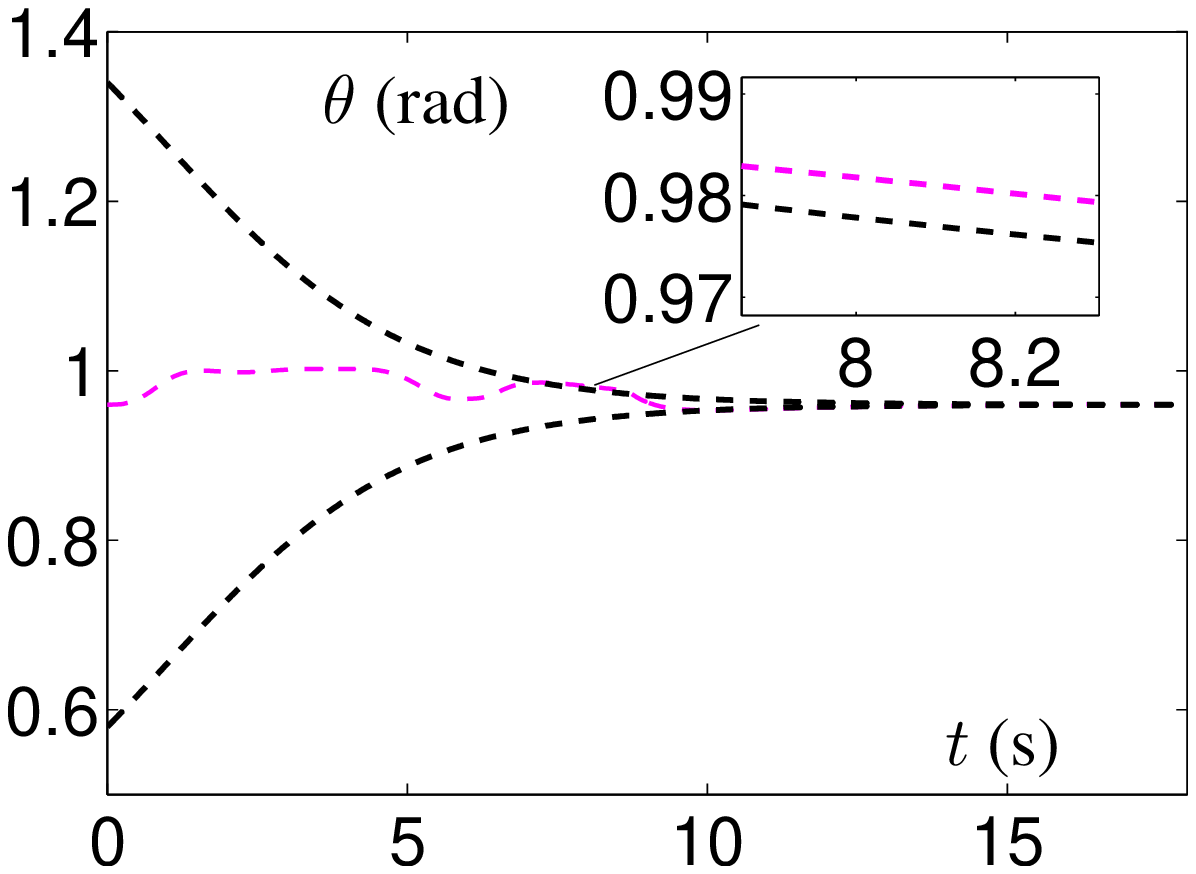}}
\subfigure[]{\includegraphics[width=0.33\columnwidth]{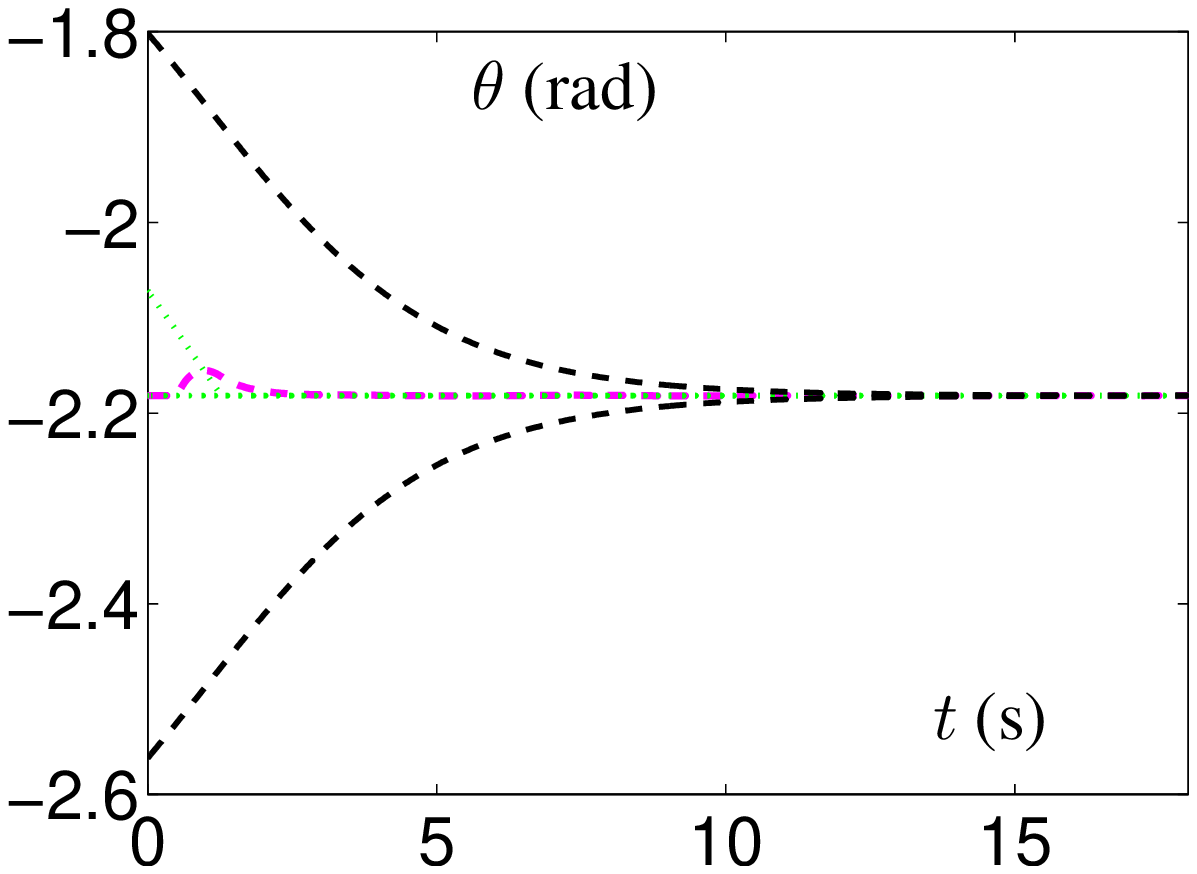}}
\subfigure[]{\includegraphics[width=0.33\columnwidth]{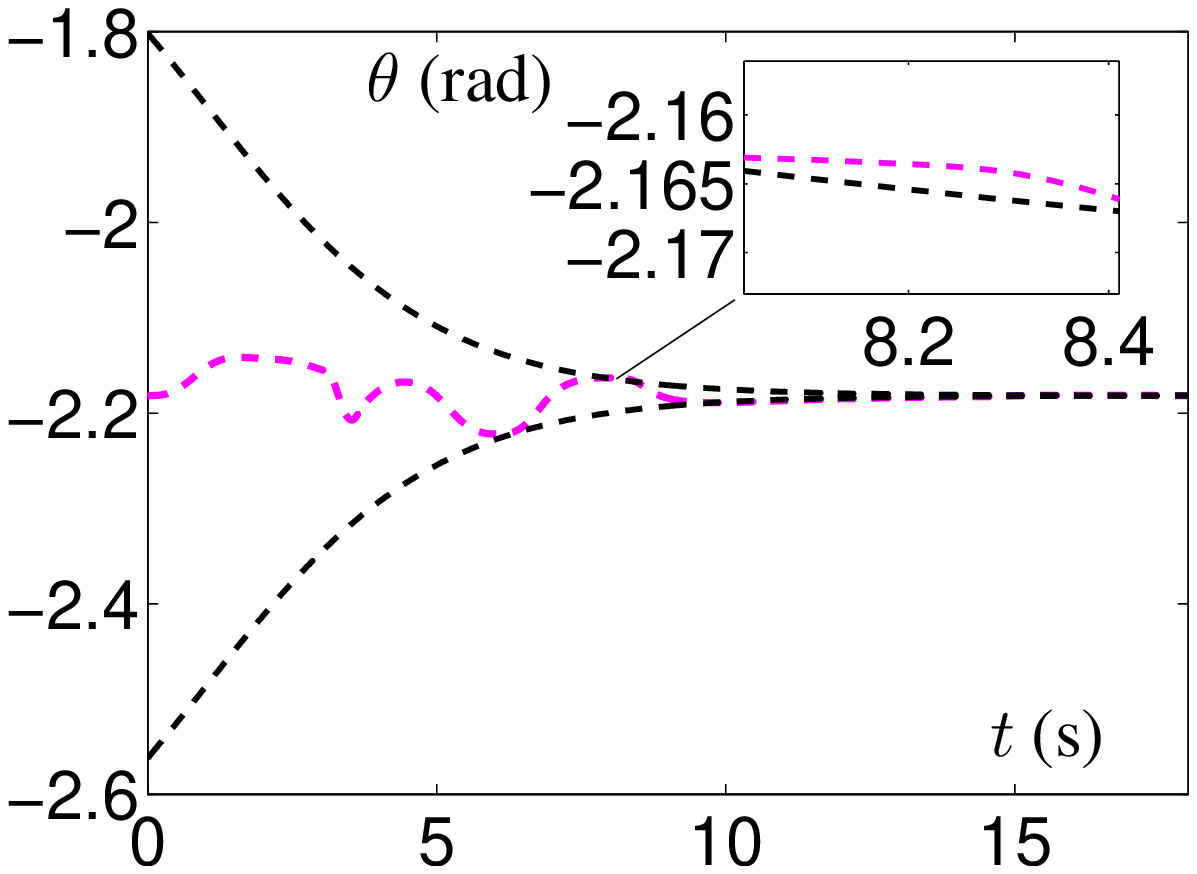}}
\caption{Comparisons between GDDF and MGDDF when $\varrho=2$, $N=4$. (a) Limit of joint $6$ with margin. (b) Limit of joint $6$ without margin. (c) Limit of joint $7$ with margin. (d) Limit of joint $7$ without margin. } \label{fig.c2n4}
\end{figure}
%%%%%% FIG-8

For joint $14$, the upper limit and lower limit both $<0$, hence
$$ \tilde{\Theta}^+ =\left\{
\begin{aligned}
&\Theta_{set}^+(t), && \tilde{\Theta}^+> \Theta_{g}^+/(2-\kappa)\\
&\Theta_{g}^+/(2-\kappa), && \tilde{\Theta}^+ \leqslant \Theta_{g}^+/(2-\kappa)
\end{aligned}
\right.
$$

$$ \tilde{\Theta}^- =\left\{
\begin{aligned}
&\Theta_{set}^-(t), && \tilde{\Theta}^-< \Theta_{g}^-/\kappa\\
&\Theta_{g}^-/\kappa, && \tilde{\Theta}^- \geqslant \Theta_{g}^-/\kappa.
\end{aligned}
\right.
$$

\subsection{Comparisons between GDDF and MGDDF}

%%%%%% FIG-9
\begin{figure}[tbp]
\centering
\subfigure[]{\includegraphics[width=0.251\columnwidth]{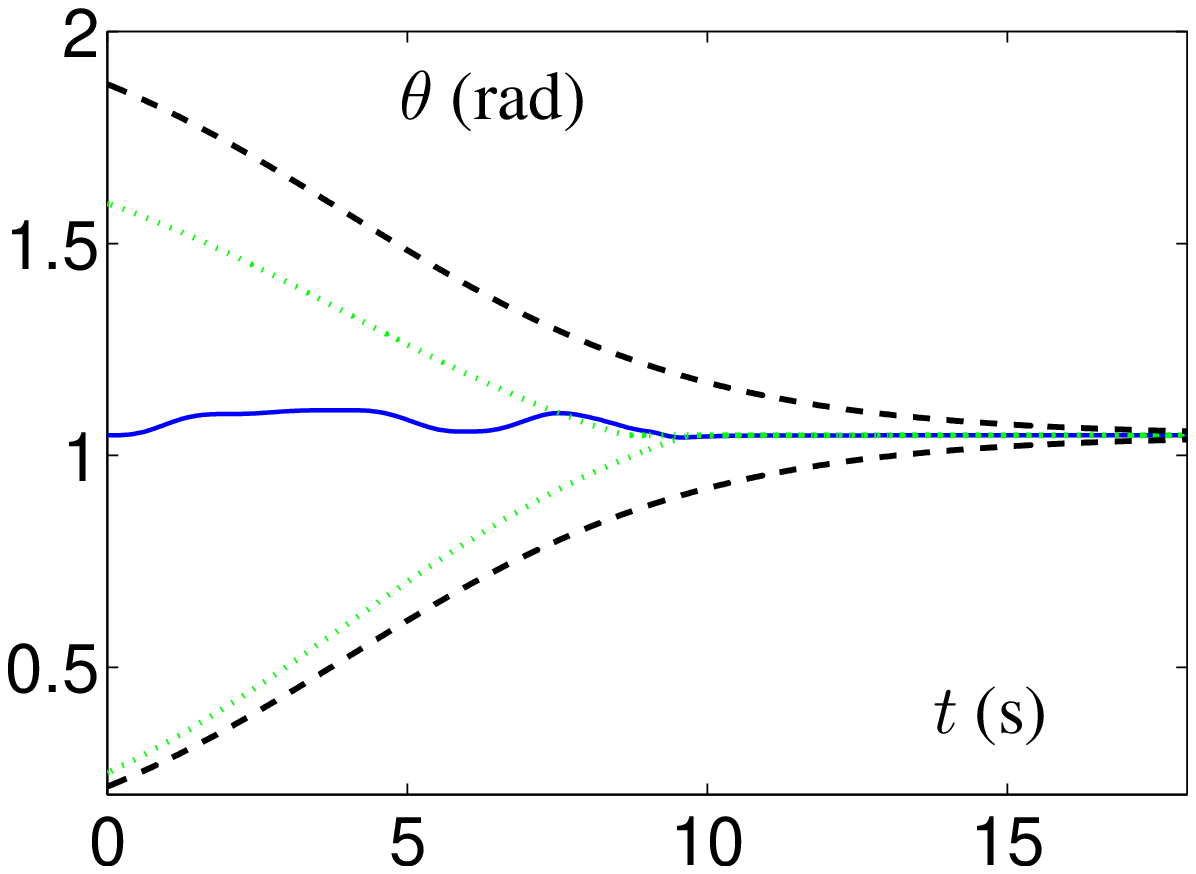}}
\subfigure[]{\includegraphics[width=0.251\columnwidth]{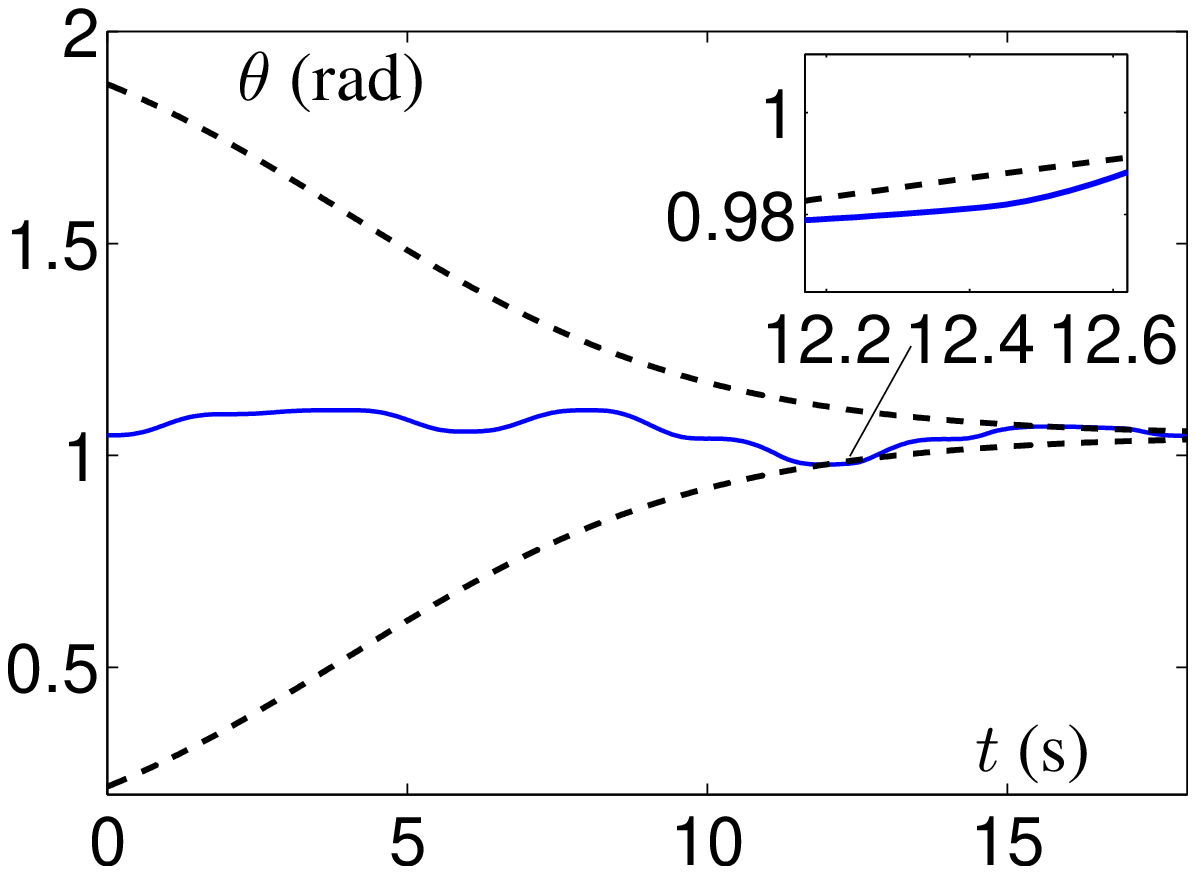}}
\subfigure[]{\includegraphics[width=0.241\columnwidth]{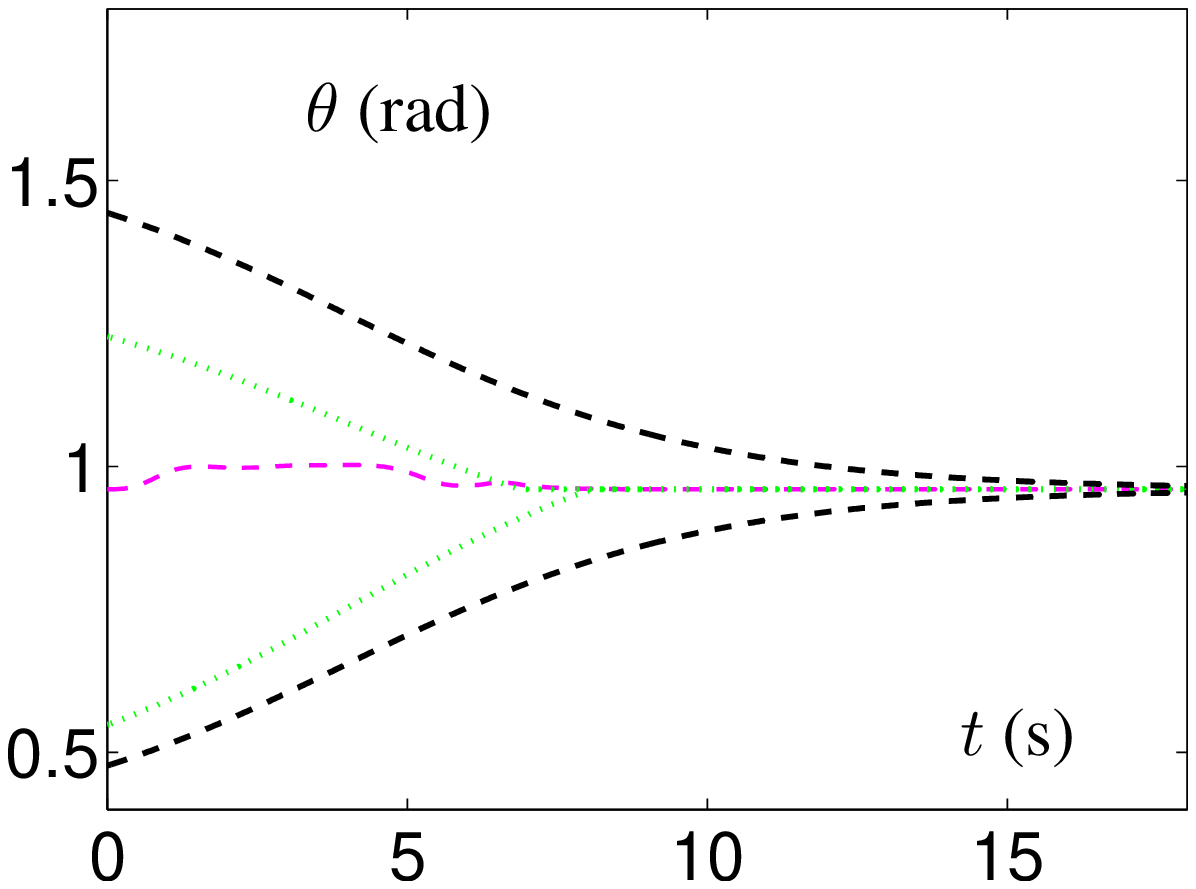}}
\subfigure[]{\includegraphics[width=0.241\columnwidth]{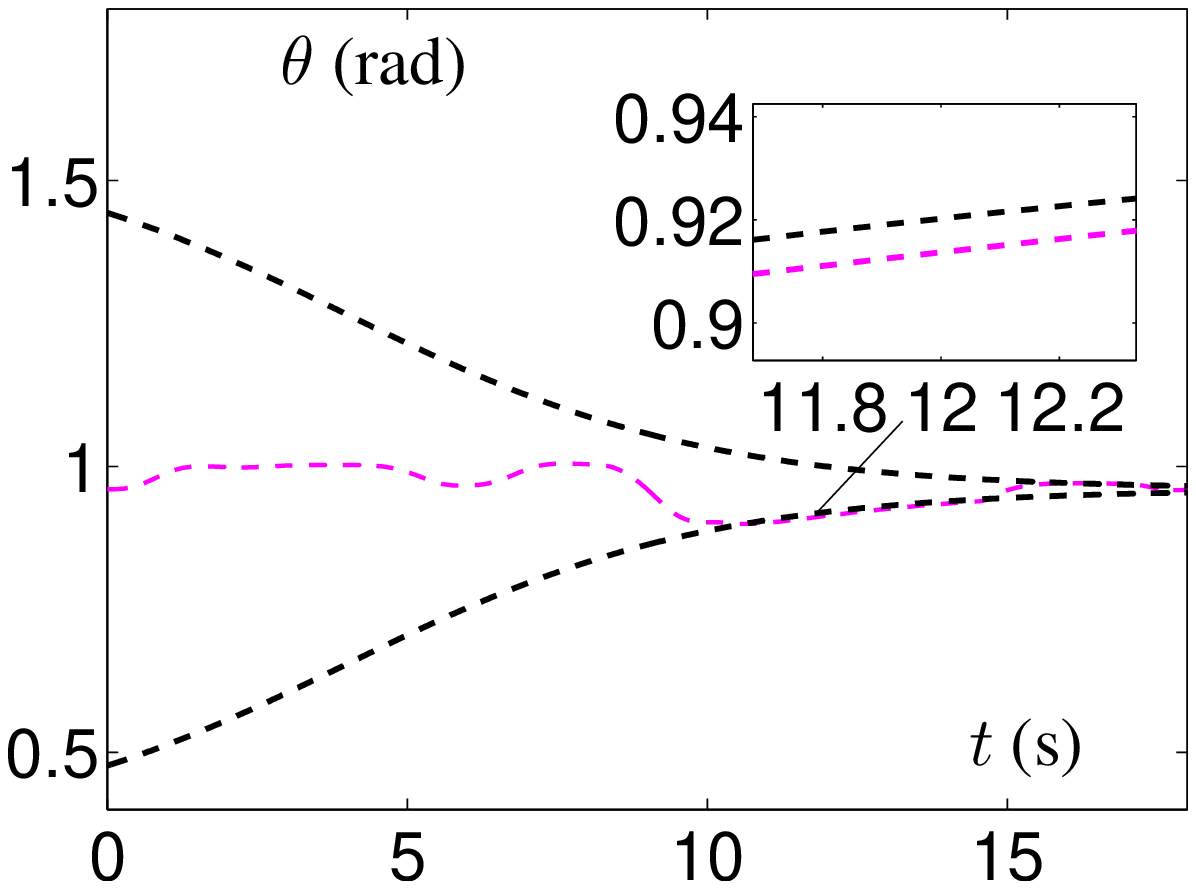}}
\subfigure[]{\includegraphics[width=0.251\columnwidth]{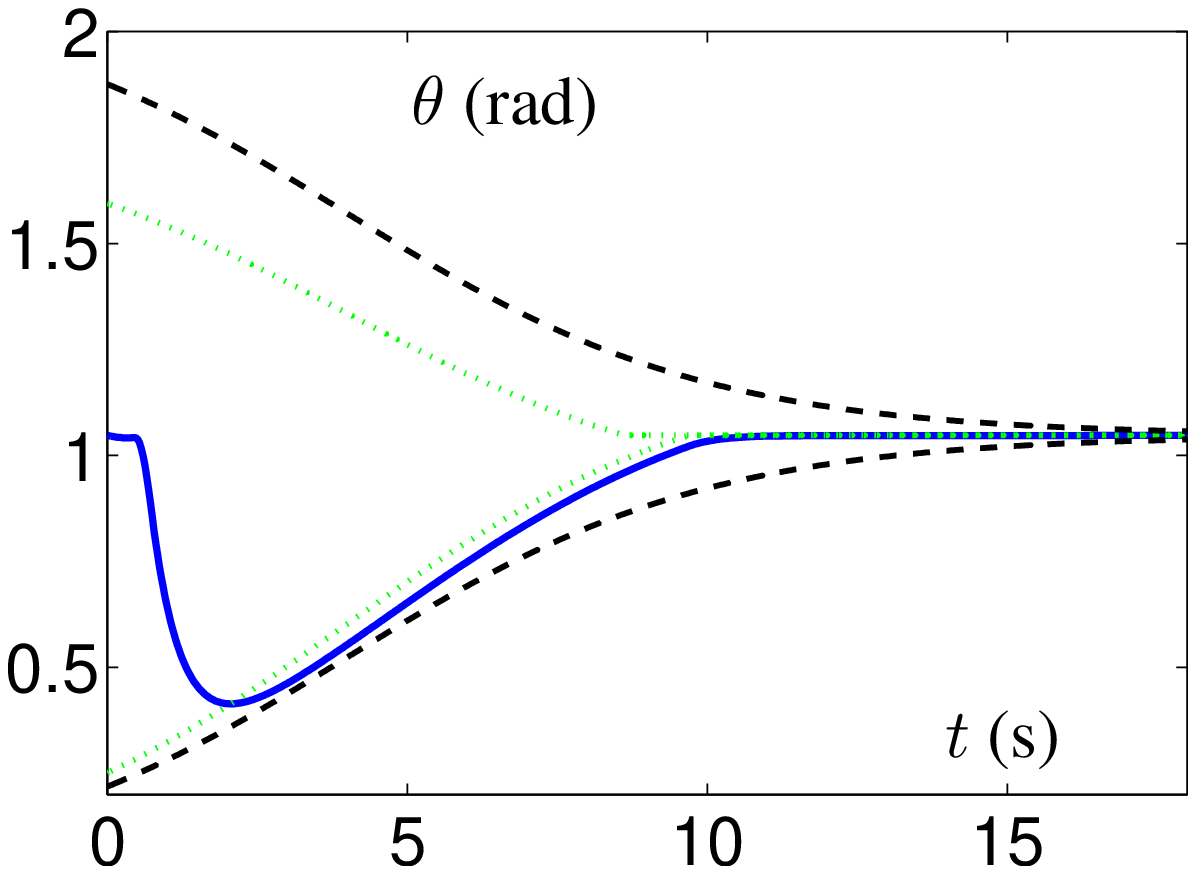}}
\subfigure[]{\includegraphics[width=0.251\columnwidth]{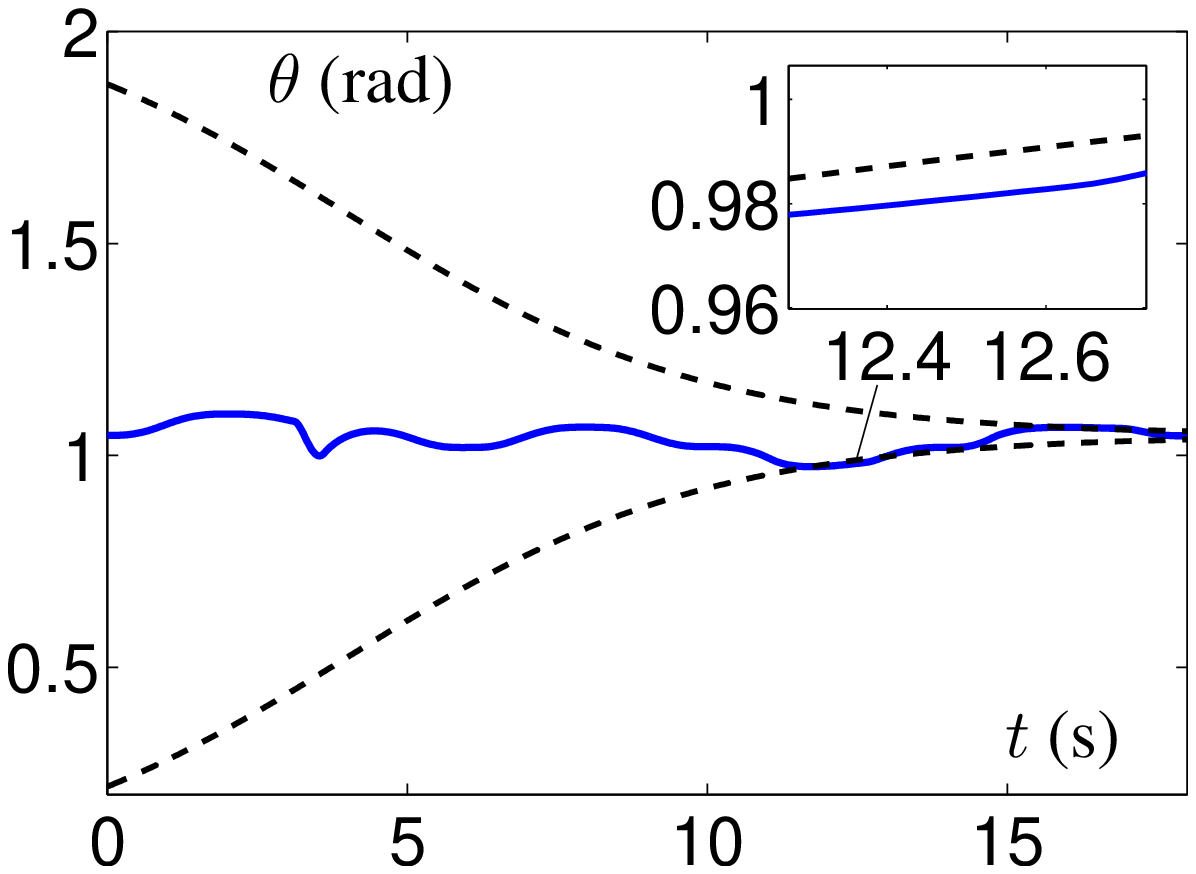}}
\subfigure[]{\includegraphics[width=0.241\columnwidth]{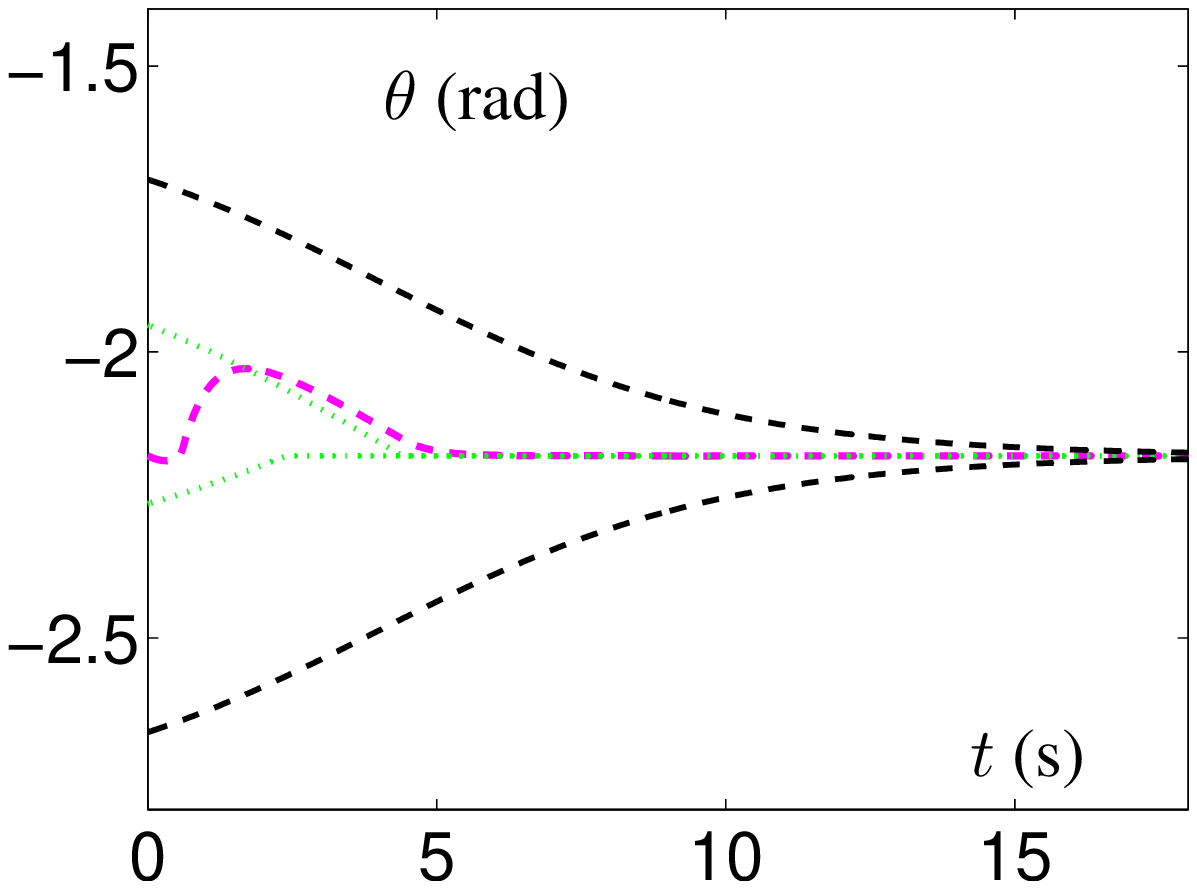}}
\subfigure[]{\includegraphics[width=0.241\columnwidth]{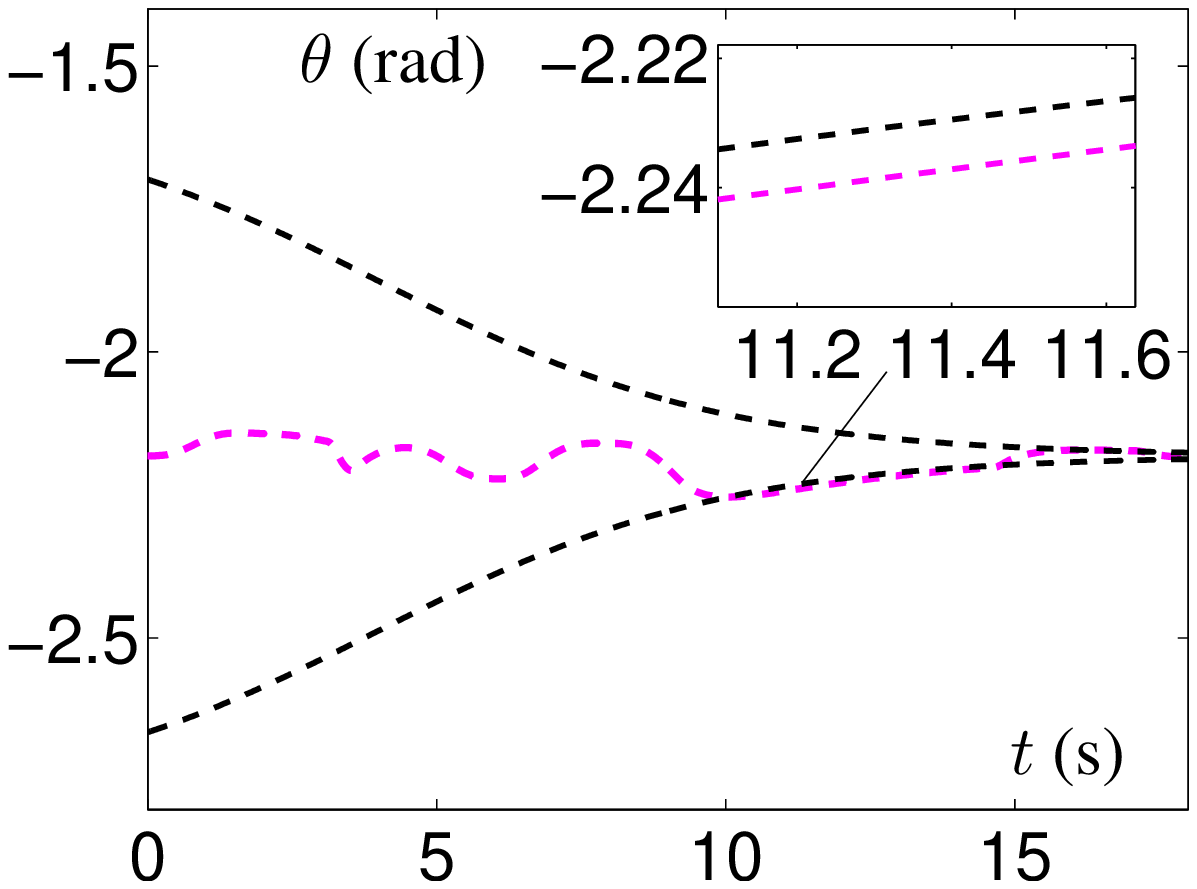}}
\caption{Comparisons between GDDF and MGDDF when $\varrho=3$, $N=1$. (a) Limit of joint $6$ with margin. (b) Limit of joint $6$ without margin. (c) Limit of joint $7$ with margin. (d) Limit of joint $7$ without margin. (e) Limit of joint $13$ with margin. (f) Limit of joint $13$ without margin. (g) Limit of joint $14$ with margin. (h) Limit of joint $14$ without margin.}\label{fig.c3n1}
\end{figure}
%%%%%% FIG-9

In this part, experiments are made to compare the proposed MGDDF scheme with the traditional GDDF scheme. Joints $6$, $7$, $13$ and $14$ are the main research objets. Similar to before, $\Theta_{g}^+$ and $\Theta_{g}^-$ of joints $6$, $7$, $13$ and $14$ are set as the same values.

The parameter $\varrho$ is set as $\varrho=2$ and $N$ is set as $N=2$. From Fig. \ref{fig.c2n2} we could see that when applying GDDF scheme, joint $6$ exceed the limit of margins, and it cannot exceed its real limit. The similar situation happens in the experiment of joint $7$.

Now parameter $N$ changes from $2$ to $3$. Joints $6$, $7$ and $14$ exceed their limits again. From Fig. \ref{fig.c2n3} (a) (c) and (e) we could see that by containing the margins, the proposed MGDDF can successfully complete the task. Similar conclusion can be drew when parameters $\varrho=2$ and $N=4$ for joints $6$ and $7$ in Fig. \ref{fig.c2n4}, parameters $\varrho=3$ and $N=1$ for joints $6$, $7$, $13$ and $14$ in Fig. \ref {fig.c3n1}, and parameters $\varrho=4$ and $N=1$ for joints $7$ and $14$ in Fig. \ref {fig.c4n1}.

\subsection{The ``Ball-catch'' Experiment}

%%%%%% FIG-10
\begin{figure}[tbp]
\centering
\subfigure[]{\includegraphics[width=0.251\columnwidth]{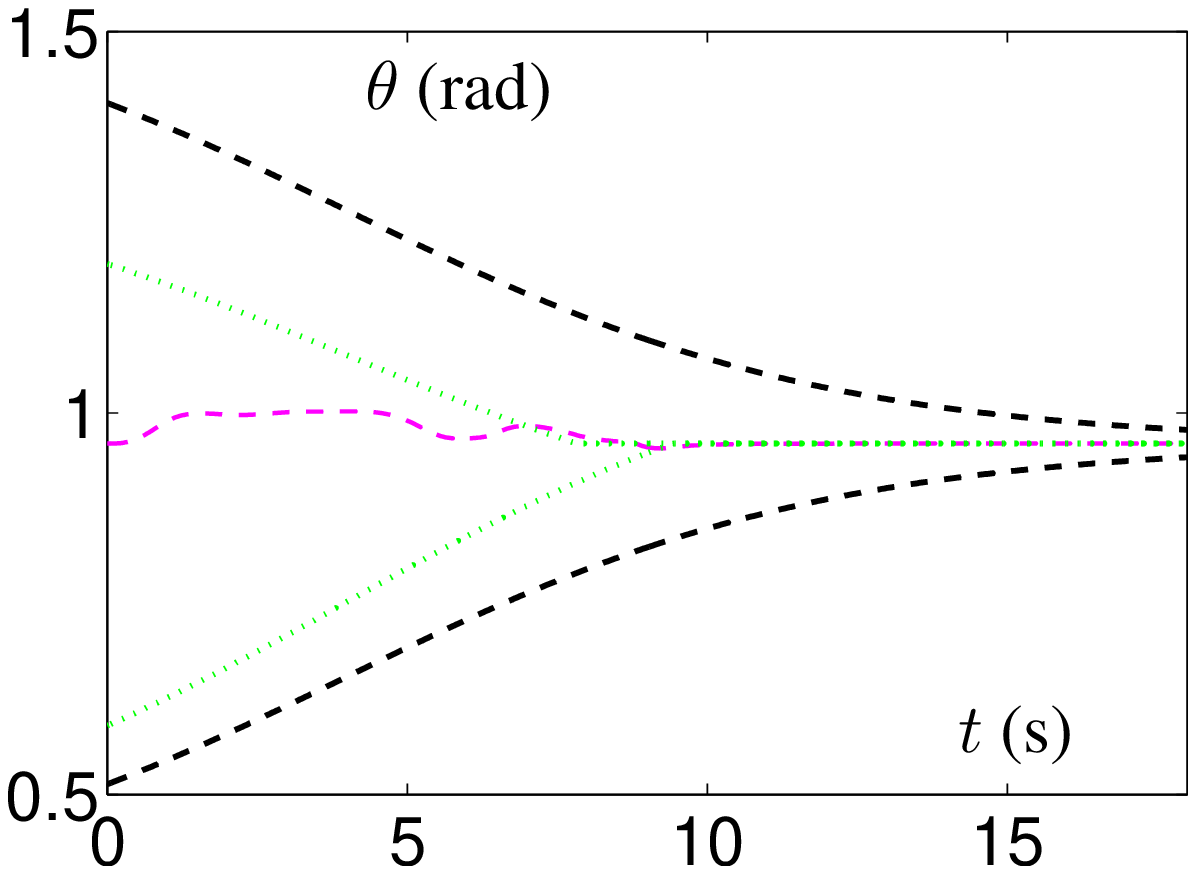}}
\subfigure[]{\includegraphics[width=0.251\columnwidth]{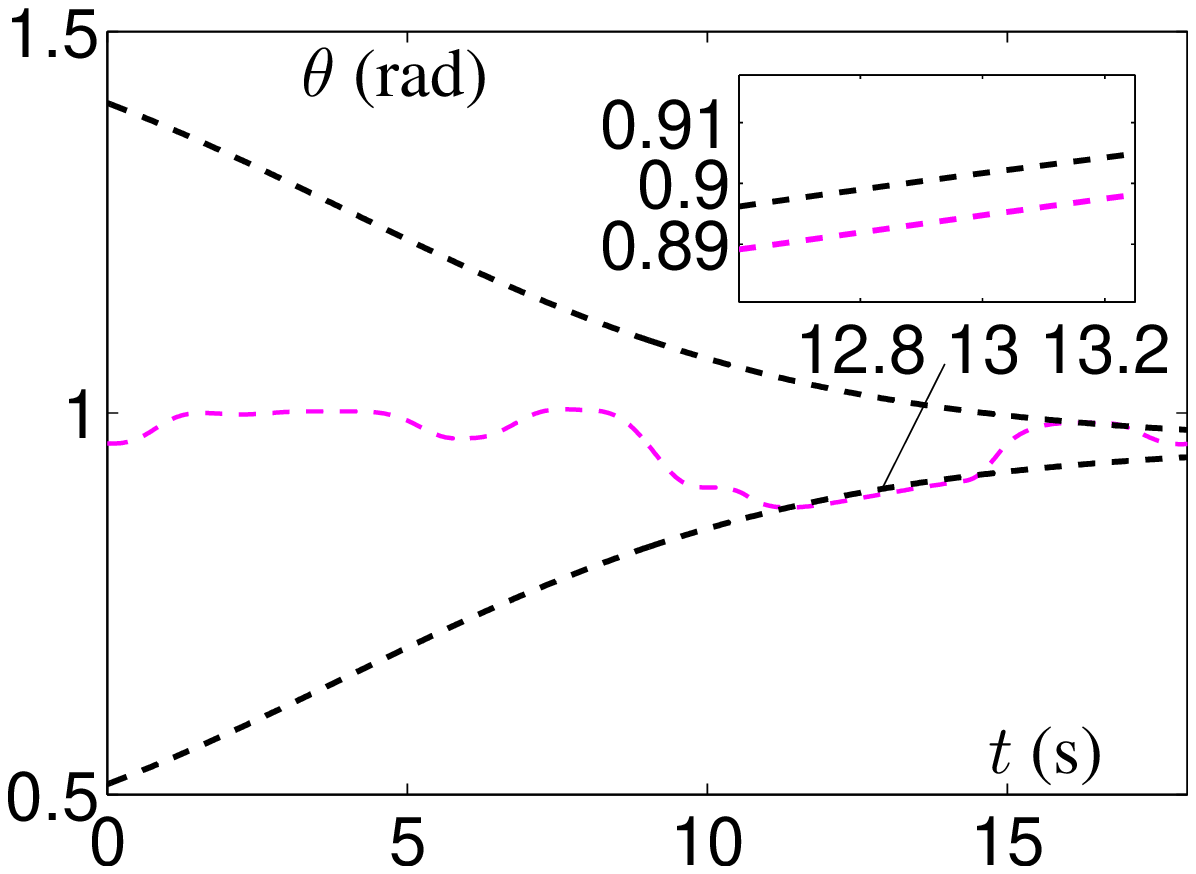}}
\subfigure[]{\includegraphics[width=0.241\columnwidth]{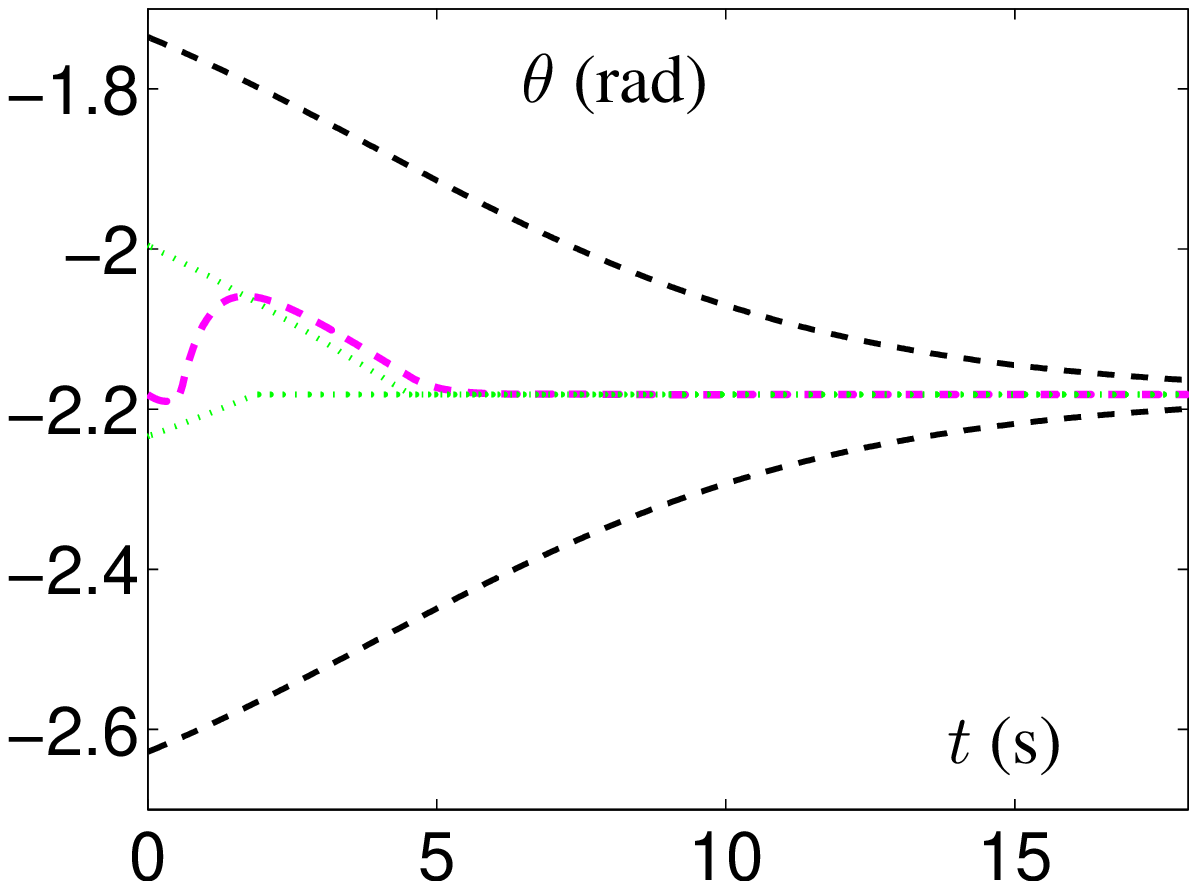}}
\subfigure[]{\includegraphics[width=0.241\columnwidth]{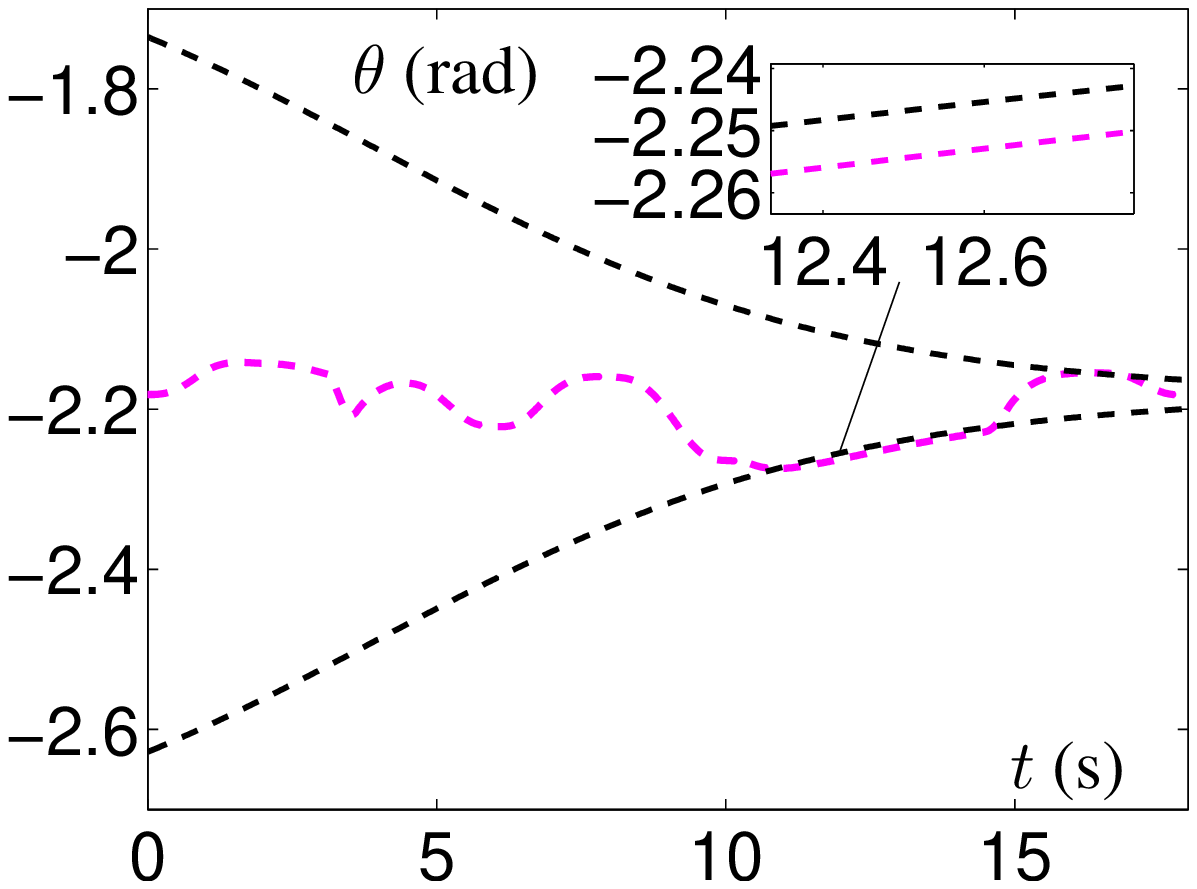}}
\caption{Comparisons between GDDF and MGDDF when $\varrho=4$, $N=1$. (a) Limit of joint $7$ with margin. (b) Limit of joint $7$ without margin. (c) Limit of joint $14$ with margin. (d) Limit of joint $14$ without margin.} \label{fig.c4n1}
\end{figure}
%%%%%% FIG-10

%%%%%% FIG-11
\begin{figure}[tbp]
\centering
\subfigure[]{\includegraphics[width=0.49\columnwidth]{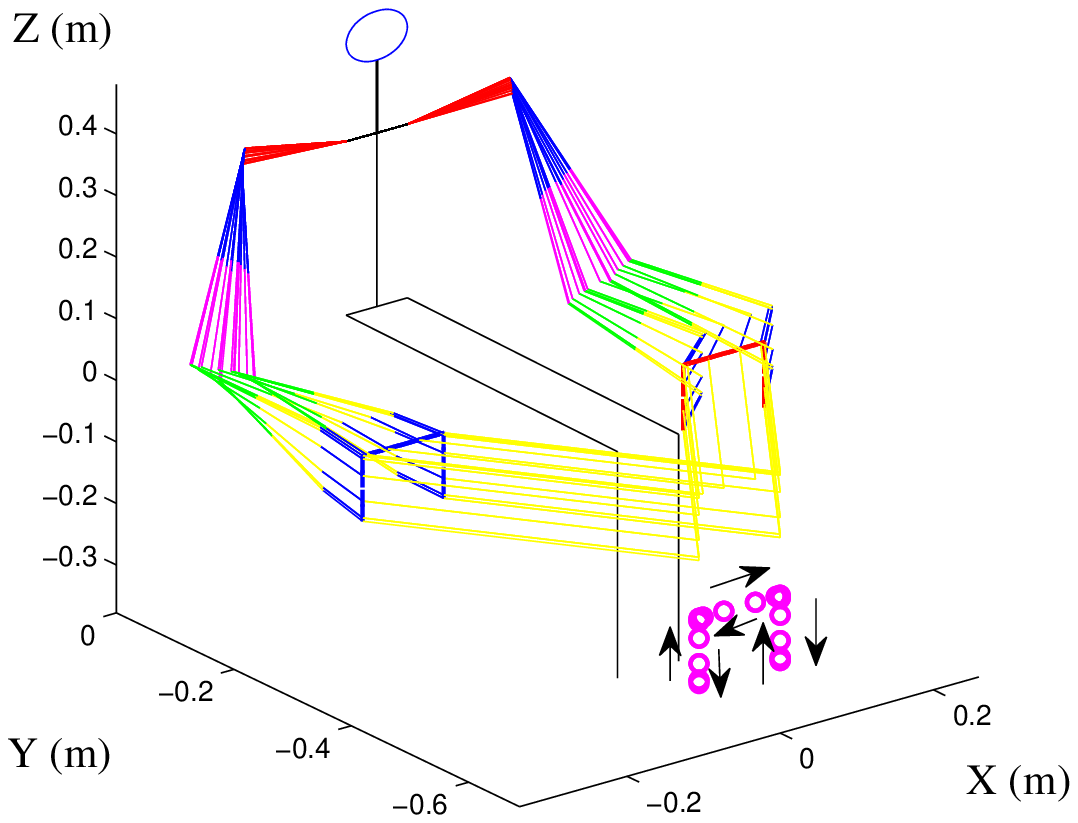}}
\subfigure[]{\includegraphics[width=0.48\columnwidth]{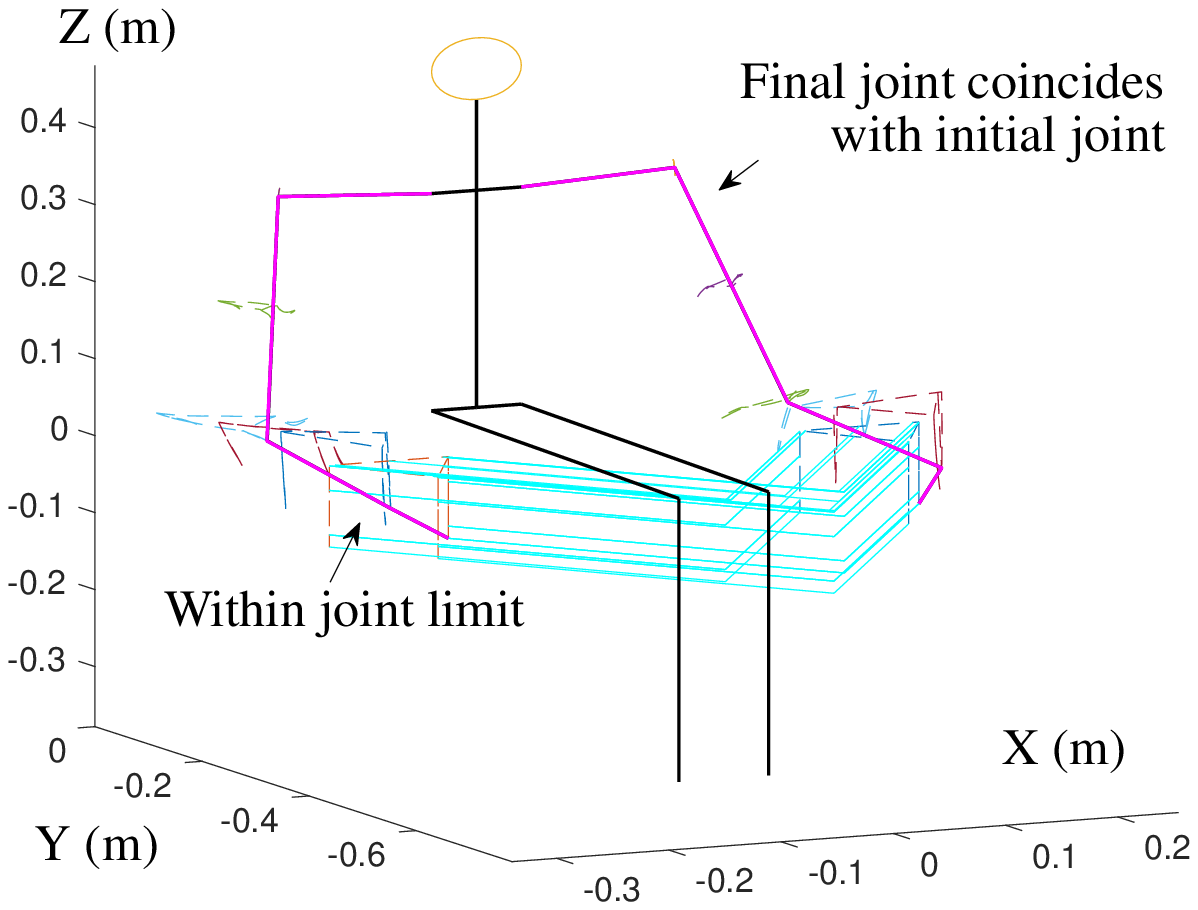}}
\caption{Motion trajectories of humanoid robot when applying to a ``ball-catch'' experiment.
(a) Tracking trajectories of each joint. (b) Joint state of the dual arms.}
\label{fig.ball}
\end{figure}
%%%%%% FIG-11

To validate the feasibility of the proposed method, a ``ball-catch'' experiment which was based on the 14 DOF humanoid is designed. For this experiment, the end-effector task is to move the ball up, left, down, and then back. Fig. \ref{fig.ball} demonstrates the motion trajectories and joint state of the dual arms of the humanoid robot in this experiment. The task can be completed and all the 14 joints maintain their joint limits successfully. All the simulative experiment results illustrated above verify the effectiveness of the proposed MGDDF method.

\section{Conclusions}
In this paper, the gesture-determined-dynamic function (GDDF) scheme is promoted and its deficiency is fixed. With consideration of three margins, a modified scheme named MGDDF is proposed for solving kinematic problems of dual arms of humanoid robots. Embed into a quadratic programming framework, the gesture of dual arms can be smoothly achieved by using MGDDF method, and the joints would not exceed their limits during the execution of tasks. Computer simulations verify the feasibility, accuracy and superiority of the proposed MGDDF method for solving motion planning and gesture determination of humanoid robots.

\section*{Nomenclature}
The full name of abbreviations proposed in this paper are listed as follows:
\begin{eqnarray*}
\begin{split}
\textsc{GDDF} ~~~~~~~~~ &\textmd{Gesture-determined-dynamic function}.\\
\textsc{MGDDF} ~~~~~ &\textmd{Modified gesture-determined-dynamic function}.\\
\textsc{QP} ~~~~~~~~~~~~~~ \ &\textmd{Quadratic programming}.\\
\textsc{DOF} ~~~~~~~~~~~ \ &\textmd{Degrees-of-freedom}.\\
\textsc{MKE} ~~~~~~~~~~ \ &\textmd{Minimum-kinetic-energy}.\\
\textsc{RMP} ~~~~~~~~~~ \ &\textmd{Repetitive motion planning}.\\
\textsc{MVN} ~~~~~~~~~ \ &\textmd{Minimum-velocity-norm}.\\
\textsc{LVI} ~~~~~~~~~~~~~ \ &\textmd{Linear-variational-inequality}.
\end{split}
\end{eqnarray*}

%%% ENTER REFERENCES IN THE FORM

\end{document}